\documentclass[preprint]{article}

\usepackage{neurips_2022}

\usepackage[utf8]{inputenc} 
\usepackage[T1]{fontenc}    
\usepackage{hyperref}       
\usepackage{url}            
\usepackage{booktabs}       
\usepackage{amsfonts}       
\usepackage{nicefrac}       
\usepackage{microtype,enumitem}      
\usepackage{xcolor}         

\usepackage{amssymb}
\usepackage{amsthm}
\usepackage{amsmath,color,comment}
\usepackage{mathrsfs}
\usepackage{todonotes,bm}

\DeclareMathAlphabet{\mathpzc}{OT1}{pzc}{m}{it} 

\DeclareMathOperator{\inj}{inj}

\newcommand{\mbC}{\mathbb{C}}

\newcommand{\mbPk}{\mathbb{P}(k)}

\newcommand{\mcB}{\mathcal{B}}
\newcommand{\mcC}{\mathcal{C}}

\newcommand{\mcM}{\mathcal{M}}

\newcommand{\mbR}{\mathbb{R}}
\newcommand{\mcS}{\mathcal{S}}

\newcommand{\kmax}{\kappa_{\max}}
\newcommand{\kmin}{\kappa_{\min}}

\DeclareMathOperator{\argmax}{argmax}
\DeclareMathOperator{\argmin}{argmin}

\newcommand{\vech}{\text{vech}}
\newcommand{\diff}{{\text d}}

\newtheorem{assumption}{Assumption}
\newtheorem{definition}{Definition}
\newtheorem{theorem}{Theorem}

\newtheorem{lemma}{Lemma}
\theoremstyle{definition}
\newtheorem{remark}{Remark}

\graphicspath{ {./images/} }

\title{Shape And Structure Preserving Differential Privacy}
\author{%
Carlos Soto \\
Department of Statistics \\
Pennsylvania State University \\
University Park, PA  \\
\texttt{cjs7363@psu.edu} 
\And
Karthik Bharath  \\
School of Mathematical Sciences \\
University of Nottingham \\
Nottingham, UK \\
\texttt{Karthik.Bharath@nottingham.ac.uk}
\And
Matthew Reimherr \\
Department of Statistics \\
Pennsylvania State University \\
University Park, PA  \\
\texttt{mreimherr@psu.edu} 
\And
Aleksandra Slavkovic \\
Department of Statistics \\
Pennsylvania State University \\
University Park, PA  \\
\texttt{sesa@psu.edu} 
}

\begin{document}

\maketitle

\begin{abstract}
  It is common for data structures such as images and shapes of 2D objects to be represented as points on a manifold. The utility of a mechanism to produce sanitized differentially private estimates from such data is intimately linked to how compatible it is with the underlying structure and geometry of the space. In particular, as recently shown, utility of the Laplace mechanism on a positively curved manifold, such as Kendall’s 2D shape space, is significantly influenced by the curvature. Focusing on the problem of sanitizing the Fr\'echet mean of a sample of points on a manifold, we exploit the characterisation of the mean as the minimizer of an objective function comprised of the sum of squared distances and develop a K-norm gradient mechanism on Riemannian manifolds that favors values that produce gradients close to the the zero of the objective function. For the case of positively curved manifolds, we describe how using the gradient of the squared distance function offers better control over sensitivity than the Laplace mechanism, and demonstrate this numerically on a dataset of shapes of corpus callosa. Further illustrations of the mechanism’s utility on a sphere  and the manifold of symmetric positive definite matrices are also presented.
\end{abstract}

\section{Introduction}
The amount of publicly available data has increased exponentially over the past decade and alongside with it the need for data privacy has emerged. As data gets increasingly more complex from scalar values to data on nonlinear manifolds, such as images, shapes and covariance matrices, there is a need for data privacy algorithms to adapt to the nonlinearity in, and preserve the geometric structure of,  the data or parameter space. Intuitively, when the geometry of the manifold significantly influences sensitivity bounds through curvature-dependent terms, one would expect a structure-preserving privacy mechanism developed directly on manifolds to have better utility that its Euclidean counterpart on higher-dimensional ambient spaces within which the manifold is embedded. This was observed in recent work on a Laplace mechanism on manifolds \citep{reimherr2021differential}, and especially for manifolds with positive curvature.

Apart from spherical, or directional, data, an archetypal example of data on a positively curved manifold arises in statistical shape analysis of planar configurations representing 2D objects; the Kendall shape space of 2D points \citep{kendall1984shape} modulo shape-preserving similarity transformations (e.g., rotations) is a Riemannian manifold with positive curvature. It is evident in this setting that the utility of a privacy mechanism will depend on how well it is compatible with shape preserving transformations of the data---sanitized versions of a shape summary of a 2D image of a bird should `look' like a bird, impervious to (global) rotation, scaling and translation.

Structure-preserving mechanisms within specific contexts have been considered before; see, for example \cite{jiang2016wishart} concerning covariance matrices;  \cite{imtiaz2018differentially,imtiaz2016symmetric,gilad2017smooth,awan2019benefits,chaudhuri2013near,biswas2020coinpress} for private principal component analysis; and \cite{sheffet2015private} for private linear regression. A general structure-preserving Laplace mechanism on manifolds was considered by \cite{reimherr2021differential}, and in \cite{awan2021structure} the relationship between the data space and the sensitivity was examined in sufficient generality. However, research on privatizing shape summaries  with theoretical guarantees is conspicuous in its absence within the privacy literature; the most relevant one we have found comes from computer vision which produces random faces but offers no differentially private guarantees \citep{karras2019style}.



%

The two-fold motivation for our paper is to a develop a privacy mechanism for statistical shape analysis, and more generally for data on manifolds, that is compatible with the geometry of the underlying space, and further, offers better control on how the curvature influences global sensitivity bounds. To this end, our main contributions are as follows.
\begin{enumerate}[leftmargin=*]
\itemsep 0em
    \item We develop an extension of the K-norm Gradient Mechanism (KNG) on $\mathbb R^d$ \citep{NEURIPS2019_faefec47} for producing sanitized private estimates under the pure differential privacy framework to the setting of Riemannian manifolds, with a focus on mean estimation.
    \item We derive a curvature-dependent upper bound on global sensitivity, which for the important case of positively curved manifolds is smaller than the corresponding one for a recently proposed Laplace mechanism \citep{reimherr2021differential}. Our numerical examples verify that the KNG mechanism on manifolds shares the powerful utility of its counterpart on $\mathbb R^d$, and this is is tied to the curvature of the manifold and not on the dimension of the ambient space. 
    \item We introduce the first, to our knowledge, differentially private shape analysis under Kendall's 2D shape space framework, and favorably compare its performance to mechanisms designed for the higher-dimensional ambient space (and not directly on the manifold) on a dataset of corpus callosa obtained from MR images.
\end{enumerate}

\section{Background} \label{sec:background}
In this section we introduce the necessary tools from differential geometry and differential privacy as well as the notation for this paper. For a thorough exposition of differential geometry and shape analysis we refer to \cite{do1992riemannian,srivastava2016functional} and for DP we refer to \cite{dwork2014algorithmic}.
\subsection{Differential geometry}
Let $\mcM$ be a complete, connected Riemannian manifold of dimension $d$. Denote by $T_m\mcM$ the tangent space at each point $m\in\mcM$ and by $T_m\mcM$, the collection $T\mcM=\{T_m\mcM:m\in\mcM \}$ of all tangent spaces, known as the tangent bundle. On the tangent space $T_m\mcM$ at each point $m$,  we can define an inner product $\langle\cdot,\cdot\rangle_m:T_m\mcM\times T_m\mcM\rightarrow\mbR$ with induced norm $\|\cdot\|_m$; the collection $\{\langle\cdot,\cdot\rangle_m:m\in\mcM\}$ is referred to as a Riemannian metric. The Riemannian metric varies smoothly along the manifold and allows us to measure distances, volumes, and angles.

For a curve, or path, $\alpha:[0,1]\rightarrow\mcM$, the vector $\alpha'(t)$  is its instantaneous velocity, and its length $L(\alpha)$ is the value $\int_0^1 \|
\alpha'(t)\|^{1/2}_{\alpha(t)} \diff t$. A curve $\alpha$ is said to be arc-length parameterised if $\|\alpha'(t)\|_{\alpha(t)}\equiv 1$ and thus $L(\alpha(0),\alpha(t_0))=t_0$. Geodesic curves are those with zero acceleration for all $t$. The distance $\rho$ between two points $p$ and $q$ is the length of the shortest path, a segment of a geodesic curve connecting the two known as the minimal geodesic: $\rho(p,q)=\inf \{L(\alpha)| \alpha:[0,1] \to \mcM;\alpha(0)=p,\alpha(1)=q\}$. 

 Given a point $p$ and a geodesic $\alpha$ with $\alpha(0)=p$, a cut point of $p$ is defined as the point $\alpha(t_0)$ such that $\alpha$ is a minimal geodesic on the interval $[0,t_0]$ but fails to be for $t>t_0$. The set of all cut points of geodesics starting at $p$ is its cut locus. The injectivity radius of $p$ is the distance to its cut locus, and the injectivity radius $\text{inj }\mcM$ of $\mcM$ is the infimum of the injectivity radii of all points in $\mcM$. 

The next two tools are necessary for moving on the manifold, moving to and from the tangent spaces, and are particularly useful for sampling from distributions on manifolds. For a geodesic $\alpha$ starting at $p$ with initial velocity $v$, the \emph{exponential map} $\exp(p,\cdot):T_p \mcM \to \mcM:$ is defined as $\exp(p,v)=\alpha(1)$. From the Hopf-Rinow theorem, on a complete manifold the exponential map is surjective. On an open ball around the origin in $T_pM$ it is a diffeomorphism onto its image outside of the cut locus of $p$, and a well-defined inverse $\exp^{-1}(p,\cdot):\mcM \to T_p \mcM$ exists, known as the \emph{inverse exponential} or logarithm map, and maps a point on $\mcM$ outside of the cut locus of $p$ to $T_pM$; thus for any $q$ outside of the cut locus of $p$, $\rho(p,q)=\|\exp^{-1}(p,q)\|_p$. 

There are many notions of curvature of a Riemannian manifold. We will mainly be concerned with \emph{sectional curvature} at a point $p$, defined to be the Gaussian curvature at $p$ of the two-dimensional surface swept out by the set of all geodesics starting at $p$ with initial velocities lying in the two-dimensional subspace of $T_p\mcM$ spanned by two linear independent vectors. 

Further, volumes of sets can be computed using the Riemannian volume form $\text{d}\mu$. In local coordinates, the coordinate-independent Riemannian volume form is defined as $\sqrt{\text{det}(g)}\diff x_1\wedge \diff x_2\cdots\wedge \diff x_d$, where
$g_{ij}=\langle \partial x_i,\partial x_j \rangle$ is the Riemannian metric tensor.  
A vector field on $\mcM$ is a differentiable mapping $\mcM \to T\mcM$ that assigns to each point $m$ on $\mcM$ a tangent vector in $T_m \mcM$. Suppose we have a smooth function $h$ defined over $\mcM$, the gradient $\nabla h$ of $h$ is the vector field defined by the relationship $\langle \nabla h(p),v\rangle_p = \partial h_p(v)$ for $p\in\mcM$ and $v\in T_p\mcM$.

\subsection{Differential privacy}
Let $D=\{x_1,\dots, x_n\}\subset\mcM$ denote a dataset of size $n$. 
In several statistical and machine learning problems, one of the most popular tools for releasing a sanitized  version of a minimizer 
\[
\hat \theta = \argmax_{x \in \mcM} U(x;D)
\]
of a utility function $U$ over $\mcM$ is the exponential mechanism introduced by \cite{4389483}, based on a density
\[
f(x;D) \propto \exp \left\{\sigma^{-1} U(x;D)  \right\},
\]
where the scale or rate parameter $\sigma$ is chosen to achieve a desired level of privacy and accounts for the sensitivity of $U$. If $\mcM$ is the Euclidean space, the density is with respect to the Lebesgue measure, and for finite or countable $\mcM$, it is with respect to the counting measure.

A modification of the exponential mechanism is the K-norm Gradient Mechanism (KNG) introduced by  \cite{NEURIPS2019_faefec47}, which turns out to have better utility quite generally. The idea is that the maximizer of $U(x;D)$ is also the point at which its gradient is zero. 
On a Riemannian manifold $\mcM$, a KNG mechanism can be constructed using the gradient vector field $\nabla U(x;D)$, where the gradient is defined with respect to the Riemannian metric, with the (unnormalised) density
\[
f(x;D) \propto \exp\{ - \sigma^{-1} \| \nabla U(x;D)\|_{x} \}
\]
defined with respect to the volume measure \citep{reimherr2021differential}, where $\|\cdot\|_x$ refers to the norm with respect to Riemannian metric at $x$ and should not be construed as the subscript $k$ as in a $k$-norm. Conditions on the sectional curvatures of $\mcM$ are typically needed to ensure that the density is integrable, and guarantees a finite normalizing constant, on general manifolds $\mcM$. Under such conditions we can introduce a definition of differential privacy similar to that of \cite{blum2005practical}.  
\begin{definition}
For $\epsilon>0$, a privacy mechanism 
satisfies $\epsilon$-differential privacy (pure differential privacy, $\epsilon$-DP) if for any pair of adjacent databases $D$ and $D'$, denoted $D\sim D'$, we have that 
$$\int_S f(x;D) \diff\mu\leq e^\epsilon \int_S f(x;D')\diff\mu$$ 
for any measurable set $S\subset\mcM$.
\end{definition}

To determine the rate parameter for the KNG mechanism, one needs to quantify the robustness or sensitivity, of the norm of gradient vector field for adjacent databases. 
\begin{theorem}\label{thm:1}
If for all neighboring $D \sim D'$ and almost all $x$ we have
\[
\|\nabla U(x;D) - \nabla U(x;D')\|_x \leq \Delta,
\]
then one can take $\sigma = 2 \Delta/\epsilon$ so that the KNG mechanism will be $\epsilon$-DP. Here $\Delta$ is referred to as the global sensitivity.
\end{theorem}

The proof of Theorem \ref{thm:1} follows directly from an application of the triangle inequality. The global sensitivity $\Delta$ plays a crucial role in determining the behavior of KNG about the optimizer of $U$. Thus far, $U$ has been a generic utility function however consideration needs to be taken for cases when $U$ does not have a global optimizer.

\section{Differentially private Fr\'echet mean estimation theory} \label{sec:Mean}
Possibly the most fundamental summary statistic is the average or mean. In a Euclidean setting, the mean has a closed form expression, however for general manifolds it is not as straightforward. The Fr\'echet mean is the natural extension of the mean and is defined as the minimizer of the variance functional 
\[
F(\cdot,D): \mcM \to \mathbb R_+, \quad F(x;D) := \frac{1}{2n} \sum_{i=1}^n \rho(x_i, x)^2,
\]
where  $\rho$ is the Riemannian distance. In general, the minimizer may not be unique, sometimes referred to as the ``set of Fr\'echet means", or may not even exist. Study of conditions that ensure existence and uniqueness has a long history (see for e.g., \cite{karcher1977riemannian} and \cite{Afsari2011}), and we thus take some necessary precautions outlined in Assumption \ref{A1}. 

\begin{assumption}\label{A1}
The dataset $D\subset B_r(p_0)$, a geodesic ball centered at $p_0$ with radius $r$, with $r<\frac{1}{2}\min \{\inj\mcM,\frac{\pi}{2}\kmax^{-1/2}\}$ and $\kmax>0$ is an upper bound on the sectional curvatures of $\mcM$.
\end{assumption}


For non-positively curved $\mcM$, $\kmax^{-1/2}$ is interpreted as $+\infty$, and Assumption \ref{A1} states that the data $D$ must be bounded. That is, the data can lie in a ball of any size as long as we  know how large the ball is as this directly affects the global sensitivity. The weaker requirement $r<1/2\min \{\inj\mcM,\pi\kmax^{-1/2}\}$ suffices to ensure existence and uniqueness of the Fr\'{e}chet mean. However, the stronger Assumption \ref{A1} is required to ensure that $(x,y) \mapsto \rho^2(x,y)$ is convex along geodesics (geodesically convex) within $B_r(p_0)$ \cite{Le2001}; for example, $\rho^2$ is geodesically convex when restricted to a ball of radius smaller than $\pi/4$ on unit spheres since $\kmax=1$. For $\rho^2$ to be \emph{strong} geodesically convex, an additional lower bound on sectional curvatures of $\mcM$ is required (see Lemma \ref{thm:lemma1}). 

For mean estimation a natural utility function is $U(x;D) = -F(x; D)$. The KNG mechanism makes use of the gradient of $U(x; D)$ at $x$ and hence the (Riemannian) gradient of $-F(x; D)$, 
which in-turn is linked to the gradient of square-distance function $x \mapsto \rho(x_i,x)^2$ for fixed $x_i$.  Under Assumption 1, each $x_i$ lies within the injectivity radius of $x$, and the gradient
$\nabla \rho(x_i,x)^2 = -2 \exp^{-1} (x,x_i).$
%
Thus
\begin{align*}
    \nabla F(x;D) & = -\frac{1}{n} \sum_{i=1}^n \exp^{-1} (x,x_i).\\
\end{align*}
The KNG mechanism, then samples from the density (with respect to the volume measure) 
\[
f(x;D) \propto \exp\left\{-\frac{1}{\sigma n}\left\| \sum_{i=1}^n \exp^{-1} (x,x_i) \right\|_x  \right\}
\]
to release a private statistical summary of the mean; note that when restricted to an open ball as per Assumption \ref{A1}, the density $f$ has a finite normalising constant that depends on $\sigma$.

In Theorem \ref{thm:sensitivity} we provide a bound for the the global sensitivity of the KNG mechanism for mean estimation on Riemannian manifolds. The bound is curvature-dependent in the sense that it depends on the radius $r$ of the ball in which the data lies, the sample size $n$, and a function of $r$ and $\kmax$.  Specifically, the bound is equivalent to the bound obtained for Euclidean spaces for non-positively curved $\mcM$ but is inflated for positively curved $\mcM$. 

\begin{theorem}\label{thm:sensitivity}
Under Assumption \ref{A1} let $D=\{x_1,x_2,\dots,x_n\}$ and $D'=\{x_1,x_2,\dots,x'_n\}$ be adjacent datasets 
. Then 
\begin{align}
\label{hmax}
     \|\nabla U(x;D)-\nabla{U}(x;D')\|_x & \leq \frac{2r(2-h_{\max}(2r,\kmax))}{n} \qquad \text{where}    \nonumber \\
     h_{\max}(s,\kmax) &:= 
    \left\{
    \begin{array}{cc}
    s \sqrt{\kmax}\cot(s\sqrt{\kmax}) ,&  \kmax >0 \thickspace;\\
    1 ,& \kmax \leq 0 \thickspace.
    \end{array}
    \right.
\end{align}
\end{theorem}
\begin{proof}
For two adjacent databases $D\sim D'$ that, without loss of generality, differ in the last element we have that
\begin{align*}
\|\nabla U(x;D)-\nabla{U}(x;D')\|_x &= \|\nabla F(x;D)-\nabla{F}(x;D')\|_x    
= \frac{1}{n} \left\|\exp^{-1} (x,x_n)- \exp^{-1} (x,x_n')\right\|_x \\
&\leq \frac{1}{n}2r(2-h_{\max}(2r,\kmax)),
\end{align*}
based on Jacobi field estimates from \cite{karcher1977riemannian}; see also \cite[Lemma 1]{reimherr2021differential}.
\end{proof}

\begin{remark}\label{remark1}
For the Laplace mechanism  on $\mcM$ with density $f(x;D) \propto e^{-\frac{1}{\sigma}\rho(x,\eta)}$ for a fixed point $\eta$, whenever $r$ is chosen as per Assumption \ref{A1}, the magnitude of inflation of global sensitivity due to curvature was shown to be $2r(2-h_{\max}(2r,\kmax))/(nh_{\max}(2r,\kmax))$ in \cite{reimherr2021differential}. The upper bound obtained by the KNG mechanism is thus strictly smaller for positively curved spaces since $0<h_{\max}(2r,\kmax)<1$; in our experiments (see \ref{ss:Spheres}), this difference in sensitivity appears to result in better utility. Although the densities for the KNG and the Laplace are not the same, they are related to the norm of a vector in $T_x \mcM$: when $\eta=\bar x$, the Fr\'{e}chet mean, since $\rho(x,\bar x)=\|\exp^{-1}(x,\bar x)\|_x$, the Laplace density uses the norm of the Fr\'{e}chet mean $\bar x$ when projected onto $T_x\mcM$; on the other hand, the KNG density uses the norm of the sample mean of data $\{x_i\}$ when projected onto $T_x\mcM$. 
\end{remark}

Our next theorem shows that the utility of KNG on manifolds, as measured by the intrinsic distance $\rho(\tilde x, \bar x)$ matches the optimal rate of $O(d/n\epsilon)$ as in Euclidean space. For the theorem however we require the following Lemma, proof of which is available in the Supplemental materials. 
\begin{lemma}\label{thm:lemma1}
In addition to Assumption \ref{A1}, assume that there exists a $\kmin \in \mathbb R$ that lower bounds the sectional curvatures of $\mcM$. Denote by $\bar x$ the Fr\'echet mean of $D$. For every $x\in B_r(p_0)$,
$$h_{\max}(2r,\kappa_{\max})\rho(\bar x,x) \leq \|\nabla U(x,D)\|_x \leq h_{\min}(2r,\kappa_{\min})\rho(\bar x,x), $$
where
\begin{equation}
\label{hmin}
     h_{\min}(s,\kappa_{\min}):=
\left\{
\begin{array}{cc}
s\sqrt{|\kappa_{\min}|}\coth({s\sqrt{|\kappa_{\min}|}})    ,& \kappa_{\min}<0 \thickspace ; \\
1     ,&  \kappa_{\min} \geq 0 \thickspace.
\end{array}
     \right.
\end{equation}
\end{lemma}


\begin{theorem}
Assume the setting of Lemma \ref{thm:lemma1}. Let $\tilde{x}$ denote a draw from the KNG mechanism restricted to $B_r(p_0)$, and $\bar{x}$ be the unique global optimizer of $U(x;D)$. We have that
$$E\left[\rho(\tilde{x},\bar{x})^2\right]=O\left(\frac{d^2}{n^2\epsilon^2}\right).$$
\end{theorem}
\begin{proof}
Recall that under Assumption \ref{A1}, $f(x;D) = C^{-1}_{\sigma}
\exp\left\{-\sigma^{-1}\left\| U(x;D) \right\|_x  \right\}$ for a finite normalizing constant $C^{-1}_{\sigma}$. 
Then
$$C_{\sigma} \geq 
\int_{B_r(p_0)} \exp\left\{ - \sigma^{-1} h_{\text{min}}(2r,\kmin) \rho( x, \bar{x}) \right\} \diff \mu(x),
$$ which due to Lemma \ref{thm:lemma1} results in the upper bound

\begin{equation}\label{ratio}
E\left[\rho(\tilde x, \bar x)^2\right] \leq
\frac{\int_{B_{r}(p_0)}\rho(x,\bar x)^2e^{-\frac{1}{\sigma}h_{\max}(2r,\kappa_{\max})\rho(x,\bar x)}\diff\mu(x)}{\int_{B_{r}(p_0)}e^{-\frac{1}{\sigma}h_{\min}(2r,\kappa_{\min})\rho(x,\bar x)}\diff\mu(x)}.
\end{equation}

Since $r \leq \frac{1}{2} \text{inj}M$, there exists a unique $v_{x} \in \exp^{-1}(\bar{x},(B_{r}(p_0)) \cap \bm B_{2r}(\bar x)$ such that $v_{x}=\exp^{-1}(\bar {x},x)$, where $\bm B_{2r}(\bar x)=\{v \in T_{\bar x}M:\|v\|_{\bar x} < 2r\}$ is the open ball of radius $2r$ centred at the origin in $T_{\bar x}M$. Denote by $S_{\bar x}$ the bounded subset $\exp^{-1}(\bar {x},(B_{r}(p_0)) \cap \bm B_{2r}(\bar{x})$ of $T_{\bar x}M$ with compact closure. With respect to the pushforward of $\diff\mu$ 
on to $T_{\bar x}M$ under the inverse exponential map at $\bar x$, the ratio in \eqref{ratio} equals
\[
\frac{\int_{S_{\bar x}} \|v_{x}\|^2_{\bar x} e^{-\frac{1}{\sigma}h_{\max}(2r,\kappa_{\max}) \|v_{ x}\|_{\bar x}} \diff(\mu\circ \exp({\bar x},\cdot))(v_{x})}
{\int_{S_{\bar x}} e^{-\frac{1}{\sigma}h_{\min}(2r,\kappa_{\min}) \|v_{x}\|_{\bar x}} \diff(\mu\circ \exp({\bar x},\cdot))(v_{x})}.
\]
The induced measure $\diff(\mu\circ \exp({\bar x},\cdot))$ on $T_{\bar x}M$ can be extended to $\mathbb R^d$ by settings its value on the complement of $S_{\bar x}$ to be zero. It is then absolutely continuous with respect to the Lebesgue measure $\text{d}\lambda$ on $\mathbb R^d$ with a Jacobian determinant uniformly bounded above and below, respectively, by constants $c_1$ and $c_2$ on (the closure) of $S_{\bar x}$.  This ensures that the above ratio is upper bounded by

\[
\frac{c_1}{c_2}\frac{\int_{\mathbb R^d} \|v_{x}\|^2_{\bar x} e^{-\frac{1}{\sigma}h_{\max}(2r,\kappa_{\max}) \|v_{x}\|_{\bar x}} \diff\lambda(v_{x})}
{\int_{\mathbb R^d} e^{-\frac{1}{\sigma}h_{\min}(2r,\kappa_{\min}) \|v_{x}\|_{\bar x}} \diff\lambda(v_{x})},
\]
which, with a change of variables and using spherical coordinates, equals
\[
\sigma^2\frac{c_1}{c_2}\frac{h_{\min}(2r,\kappa_{\min})}{h_{\max}(2r,\kappa_{\max})^3} \left[\int_{0}^\infty z^{d-1}e^{-z}\diff \lambda(z)\right]^{-1} \int_{0}^\infty z^{d+1} e^{-z}\diff \lambda(z)=O\left(\frac{d^2}{n^2\epsilon^2}\right),
\]
since the curvature-dependent terms $h_{\min}(2r,\kappa_{\min})$ and $h_{\max}(2r,\kappa_{\max})$, defined in \eqref{hmax} and \eqref{hmin}, are both positive and finite under Assumption \ref{A1}, $\sigma$ is as in Theorems \ref{thm:1} and \ref{thm:sensitivity}, and the integrals in the numerator and denominator equal $(d+1)!$ and $(d-1)!$, respectively.
\end{proof}

\section{Examples}\label{sec:examples}
In this Section, we consider two simulated examples and a real data example on 2D shapes. For the former, we consider the positively curved unit $d$-sphere and the set of symmetric positive definite matrices (SPDM), which when equipped with an affine invariant metric is negatively curved. 
\subsection{Spheres}\label{ss:Spheres}
Let $\mcS^d_{\kappa}$ denote the $d$-dimensional sphere with radius $\kappa^{-1/2}$. The sphere equipped with the induced metric from $\mbR^{d+1}$, the canonical metric, has constant positive curvature $\kappa$. The tangent space at $p\in\mcS^d_{\kappa}$ is then $T_p\mcS^d_{\kappa}=\{v\in\mbR^{d+1} |\langle v,p\rangle = 0  \}$. The exponential map, defined on all of $T_p\mathcal S_\kappa^{d}$, is given by
$ \exp(p,v) = \cos(\|v\|) p + \kappa^{-1/2} \sin(\|v\|) v/\|v\|$ and $\exp(p,\bm 0)=p$. The set of points at a distance at least $\pi$ from $p$ constitutes its cut locus, which then is the singleton $\{-p\}$. With $\theta:=\rho(p,q)= \cos^{-1}(\langle p,q\rangle)$, the inverse exponential map $\exp^{-1}(p,q) = \theta(\sin(\theta))^{-1} (q-\cos(\theta) p)$ at $p$ is hence defined only within the open ball around $p$ with radius $\pi/(2\kappa^{1/2})$.

Next, we consider the utility of KNG on a manifold and compare it to other sanitization techniques. We generate random samples $D$ from $S_1^2$ and compute the Fr\'echet mean $\bar{x}$, both as described in the Supplemental material. We set $\epsilon=1$ and sanitize $\bar{x}$ with three separate methods; first with the proposed method KNG on manifolds to produce $\tilde{x}_{KNG}$, second with the Laplace on manifolds as in \cite{reimherr2021differential} to produce $\tilde{x}_{L}$, and lastly embedding $\bar{x}$ into $\mbR^3$ and privatizing with the Euclidean Laplace to produce $\tilde{x}_{E}$. The latter almost surely
will not be on the sphere, however, since the privacy guarantees are invariant to post-processing, we project the estimate back onto the sphere by normalizing as $\tilde{x}_E\rightarrow \tilde{x}_E/\|\tilde{x}_E\|$. 

We compute several such replicates at different sample sizes and display the utility comparison in the first (left) panel of Figure \ref{fig:Utility_Mean_Sphere}, where the utility is measured using the  average Euclidean distance $\|\bar{x}-\tilde{x}\|$ 
such that $\tilde{x}$ is a the respective sanitized estimate. 
We see that adding noise in the ambient space with the Euclidean Laplace adds the most noise, which may be attributed to the need to sanitize over an extra dimension. After post-processing $\tilde{x}_E$ by projecting onto the sphere, the Euclidean Laplace mechanism does have better utility than its manifold counterpart, but this is not unexpected since the sensitivity of the manifold Laplace for positively curved manifolds is inflated compared to the Euclidean rate of $2r/n$ \cite{reimherr2021differential}. Lastly, our proposed mechanism has the best utility in this comparison which may be attributed to its sensitivity being strictly less than the sensitivity of the Laplace on the manifold for positively curved manifolds. Further, our approach will always produce a private summary which is on the manifold and does not require any post-processing.

\begin{figure}
\centering
    \begin{tabular}{cc}
         \includegraphics[height = 1.65in]{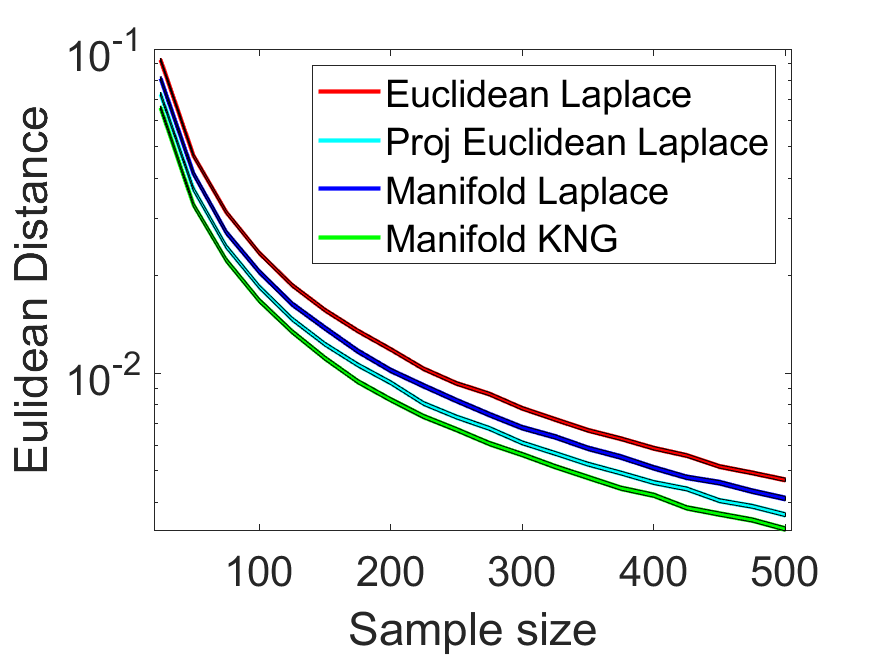}
         &  
         \includegraphics[height = 1.65in]{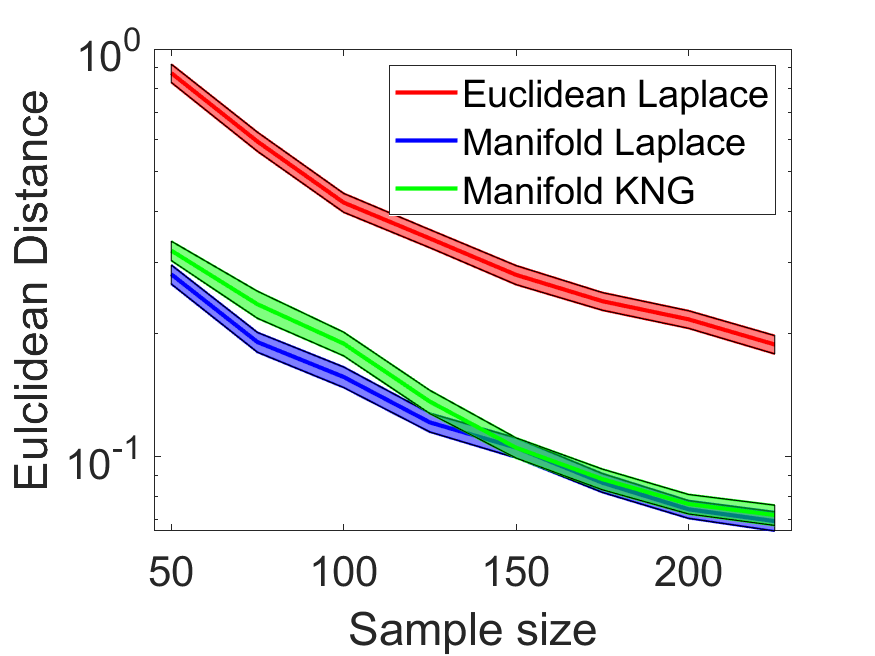}
    \end{tabular}
    \caption{Utility measured using 
    average Euclidean distance between the Fr\'echet mean $\bar x$ and its sanitized version $\tilde x$, when $\mcM$ is the unit sphere $\mathcal S^2_1$ in two dimensions (left) and  $k\times k$ SPD matrices $\mathbb P(k)$ (right), under the following frameworks: (i) Manifold KNG; (ii) Euclidean Laplace by embedding $\bar x$ into the ambient space; (iii) Manifold Laplace on the manifold; and additionally, with (iv) Projected Euclidean Laplace for the unit sphere. For the Euclidean Laplace $\bar x$ was embedded into $\mathbb R^3$ for $\mathcal S^2_1$ and within $k \times k$ symmetric matrices for $\mathbb P(k)$. 
    For each sample size, 10000 replicates were used for $\mathcal S^2_1$, whereas 500 were used for $\mathbb P(k)$. Shaded regions represent the average distance $\pm2\text{SE}$, where the Euclidean distance for $\mathbb P(k)$ is $\|\vech(\bar{x})-\vech(\tilde{x})\|$. 
    }
    
    \label{fig:Utility_Mean_Sphere}
\end{figure}
\subsection{Symmetric positive-definite matrices}\label{ss:SPDM}

Denote by $\mbPk$ the $k(k+1)/2-$dimensional manifold of $k\times k$ symmetric positive-definite matrices equipped with the affine-invariant Rao-Fisher metric $\langle v , u\rangle_p = \text{Tr}(p^{-1} u p^{-1} v) $, where $u,v\in T_p\mbPk=Sym_k$ are symmetric matrices. With this Riemannian metric $\mbPk$ has negative sectional curvature everywhere with exponential map $\exp(p,v)=p^{1/2}\text{Exp}\left(p^{-1/2}vp^{-1/2}\right)p^{1/2}$ and globally defined inverse exponential map $\exp^{-1}(q,p) = q^{1/2}\text{Log}\left(q^{-1/2}p q^{-1/2}\right)q^{1/2}$, where $\text{Exp}(\cdot)$ and $\text{Log}(\cdot)$ are the matrix exponential and logarithm, respectively. The distance between $q$ and $p$ in $\mbPk$ is thus $\rho(q,p)=\|\exp^{-1}(q,p)\|_q =\text{Tr}[\text{Log}(q^{-1/2}pq^{-1/2})^2]^{1/2}$.

For the simulations we set $k=2$. We generate random samples $D\subset \mbPk$ by sampling from the Wishart distribution as discussed in the Supplemental materials. We compute the Fr\'echet mean $\bar{x}$ and sanitize it with three separate approaches: (i) we generate a private mean $\tilde{x}_{KNG}$ by sanitizing on $\mbPk$ using the proposed approach KNG, (ii) we generate a private mean $\tilde{x}_{L}$ on $\mbPk$ with the Laplace distribution as in \cite{reimherr2021differential,hajri2016riemannian}, and (iii) we embed $\bar{x}$ into $Sym_k$ the space of symmetric matrices, represent $\bar{x}$ as a vector $\text{vech}(\bar{x})\in\mbR^3$, sanitize by sampling from the Euclidean Laplace to produce $\text{vech}(\tilde{x}_E)$, and lastly we revert the vectorization to obtain $\tilde{x}_E$. There is no guarantee that $\tilde{x}_E$ will remain in $\mbPk$ and further without a unique projection to $\mbPk$ since it is an open cone within $\mathbb R^{k(k+1)/2}$. 

Similarly to Section \ref{ss:Spheres}, in the second panel of Figure \ref{fig:Utility_Mean_Sphere} we display an average utility comparison of the privatization techniques over 500 replicates with respect to the distance $\|\text{vech}(\bar{x})-\text{vech}(\tilde{x})\|$, where $\tilde x$ is the sanitized estimate corresponding to one of the three approaches. 
We see that our approach has better utility compared to the Euclidean approach and comparable utility to the Laplace on $\mbPk$. The latter is not entirely surprising since the Laplace is equivalent to KNG for mean estimation in Euclidean space \citep{NEURIPS2019_faefec47}. Sampling from the Laplace on $\mbPk$ is fairly simple since it was thoroughly studied by \cite{hajri2016riemannian} and has nearly a closed form sampler; sampling from KNG on $\mbPk$ however is not as straightforward and we employed an Metropolis-Hastings algorithm which may account for its inconsistent behavior compared to the Laplace. However, our proposed method has better utility than the Euclidean approach, which is designed on the higher-dimensional ambient space.

\subsection{Kendall's 2D shape space}\label{ss:K2D}
Statistical shape analysis is a relatively recent field dating back to the seminal paper by \cite{thompson1942growth} where shapes of animals, such as fish, were shown to differ in geometric transformations such as a shear. Since this conception there have been many branches of shape analysis that have arisen such as Kendall's shape space \citep{kendall1984shape}, large deformation diffeomorphic metric mapping (LDDMM) \citep{GrenanderMiller}, and elastic shape analysis \citep{srivastava2016functional}. No matter the choice, shape analysis has demonstrated to be widely applicable in the medical field (\cite{ercan2012statistical,li2014shape}), computer vision \citep{jimenez2000survey,sharon20062d}, and functional data analysis (\cite{harris2021elastic,zhang2020shape}). By ``shape" of an object in two dimensions, we refer to the intrinsic geometric property of a set of points on the plane (representing the object) that remains unchanged under 
 similarity transformations such as translation, rotation and scale \cite{kendall1984shape}, and additionally on reparameterisations if an outline curve representation is used \cite{srivastava2016functional}. Of the many areas of shape analysis available (for e.g., the theory of deformable templates \citep{trouve2005local} and Large deformation diffeomorphic metric mapping  \citep{GrenanderMiller}) we consider the Kendall shape space of two-dimensional landmark configurations \cite{kendall1984shape}. 
 
Consider a set $x=\{x_j\} \in \mbC^k$ of labelled $k$ points on a 2D object, known as landmarks, in the complex plane. If the object has been extracted from a densely sampled outline curve (for e.g., when extracted/segmented from a 2D image), then the parameterisation of the curve induces the labelling through an ordering of the points. Labelling thus establishes a correspondence between points on different objects, and is considered to be fixed.     

The shape of $x$ is what remains once translation, scaling and rotation variabilities are removed or accounted for. Translation is removed by by transforming $x \to x-\frac{1}{k}\sum_{j=1}^k x_j$ resulting in the complex $(k-1)-$dimensional hyperplane $\mcC=\{x \in \mathbb C^k\backslash 0|\frac{1}{k}\sum x_j=0\}$. 
Scaling and rotation of $x$ amounts to multiplying by a complex number $re^{\text{i}\theta}$. The shape of $x$ then can be considered as the curve $(-\pi,\pi]\ni\theta \mapsto e^{\text{i}\theta}u$ on the complex unit ($k-1$)-dimensional sphere $\mathbb C S^{k-1}$ in $\mcC$, where $u=x/\|x\|\in \mcC$ (or the real sphere of dimension $2k-3$), $\|x\|=\sqrt{x^* x}$ and $x^*$ the complex conjugate of $x$. The shape space is thus identified with the compact complex projective space $\mathbb CP^{k-2}$ of dimension $k-2$ following the scaling $x \to x/\|x\|, x \in \mcC$. Therefore, the geodesic shape distance between landmark configurations $x$ and $y$ with corresponding centred and scaled versions $p$ and $q$ is
$\rho(x,y)=\inf_{\theta \in (\pi,\pi]}\cos^{-1}\left(|e^{-\text{i}\theta}pq^*|\right)$; thus, the injectivity radius of the shape space is $\pi/2$. 

Let $x$ be a centred and scaled configuration. Minimizing unit-speed geodesics starting at $x$ and initial velocity $v$ in the shape space are isometrically identified with unit-speed geodesics $(-\pi/2,\pi/2] \ni s \mapsto x \cos (s)+v \cos \sin(s)$ on $\mathbb C S^{k-1}$ wherein $v$ satisfies $v\bm 1_k=0$ in addition to $xv^*0$, and $\bm 1_k$ is the vector of ones\citep[Chapter 6]{ShapeBook}; such geodesics are known as `horizontal' geodesics. Consequently, the exponential map on the shape space is given by 
the corresponding one on the sphere with the additional condition on the velocity vectors. 
The inverse exponential map at $x$ exists within a ball of radius smaller than $\pi/2$ in the shape space, and is thus given by $\exp^{-1}(x,y)=\theta\|y-\text{Proj}_x(y)\|$, where $\text{Proj}_x(y):=x (y^*x)$ is the projection of $y$ onto $x$ and $\theta = \cos^{-1}|x^*y|$. 
The (complex) holomorphic sectional curvature of the 2D Kendall's shape space is constant and equals 4 \citep{ShapeBook}. 

\begin{figure}
    \centering
    \begin{tabular}{@{}c@{} @{}c@{} @{}c@{}}
        \begin{tabular}{|c|c|}
        \hline
            {\includegraphics[trim = 40 5 20 5, clip, height=1.5in]{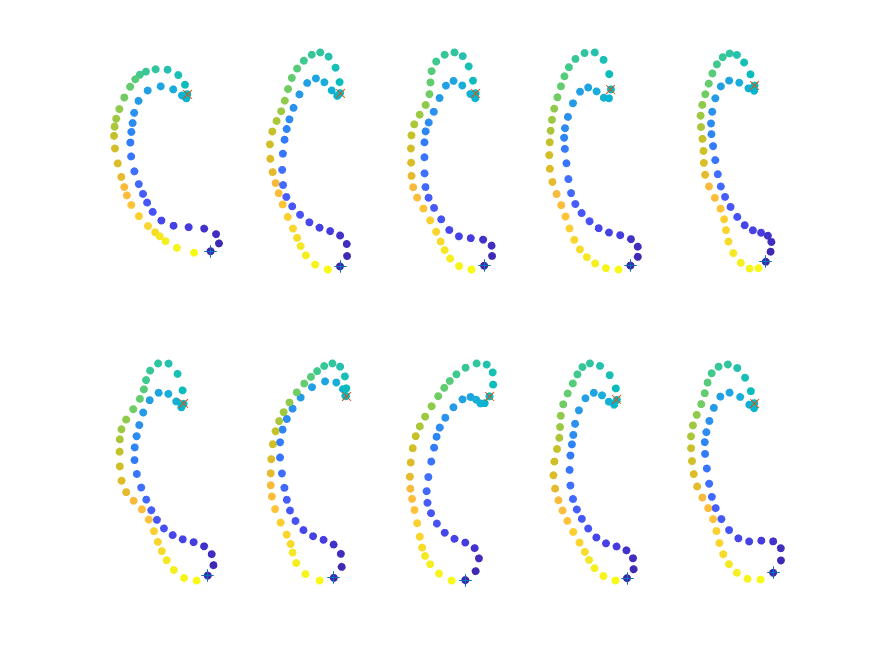}} &  
            {\includegraphics[trim = 150 15 120 5, clip, height=1.5in]{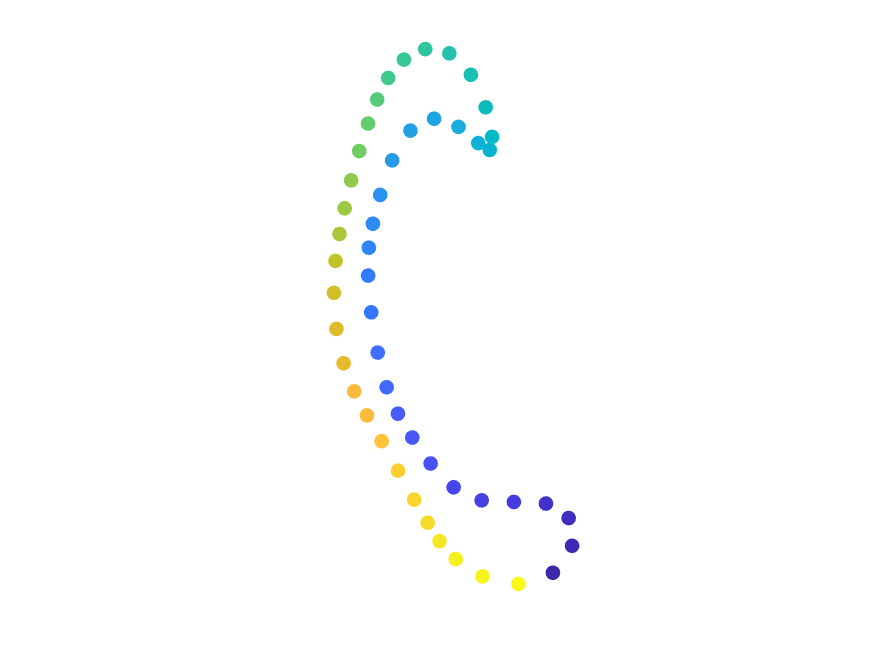}} \\
        \hline
        \end{tabular}
        
        \begin{tabular}{|@{}c@{} @{}c@{} @{}c@{} @{}c@{} @{}c@{} @{}c@{} |}
        \hline
            {\includegraphics[trim = 150 15 120 5, clip, height=.6in]{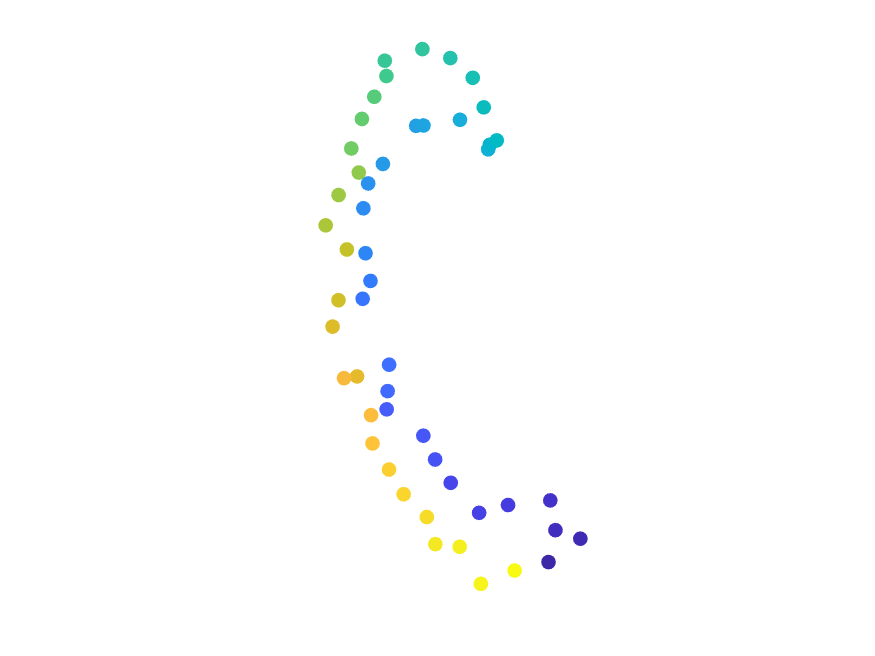}} &
            {\includegraphics[trim = 150 15 120 5, clip, height=.6in]{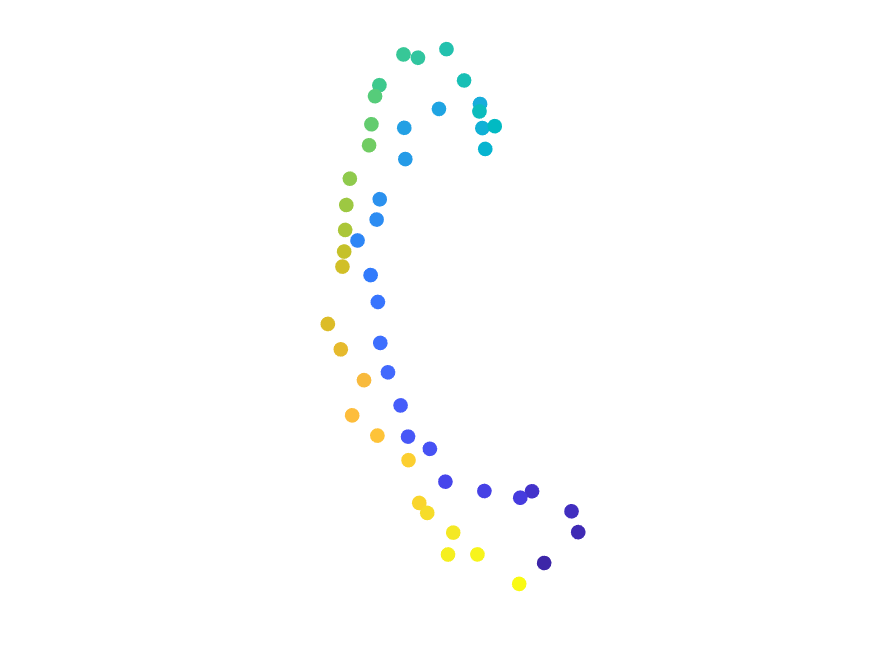}} &
            {\includegraphics[trim = 150 15 120 5, clip, height=.6in]{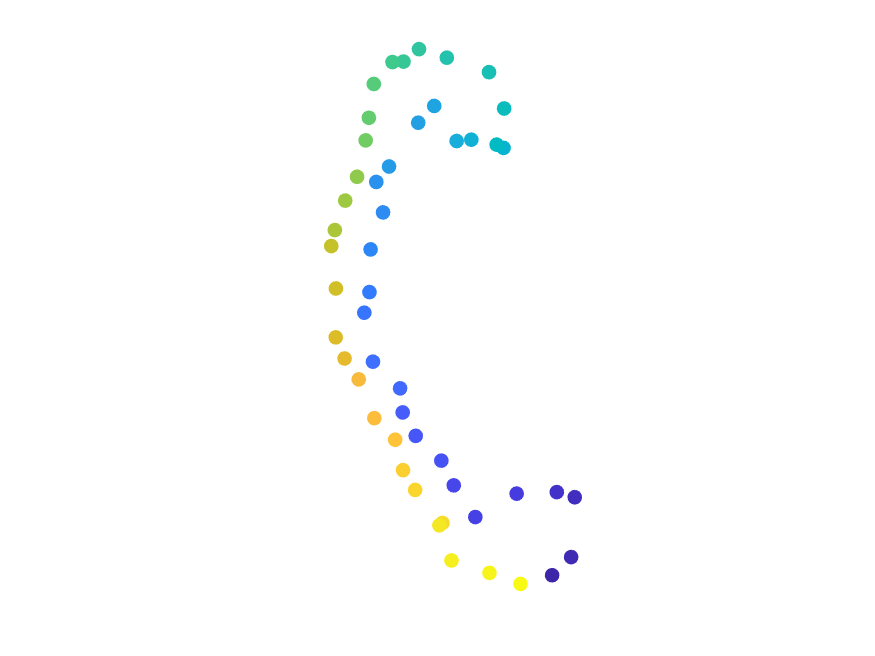}} &
            {\includegraphics[trim = 150 15 120 5, clip, height=.6in]{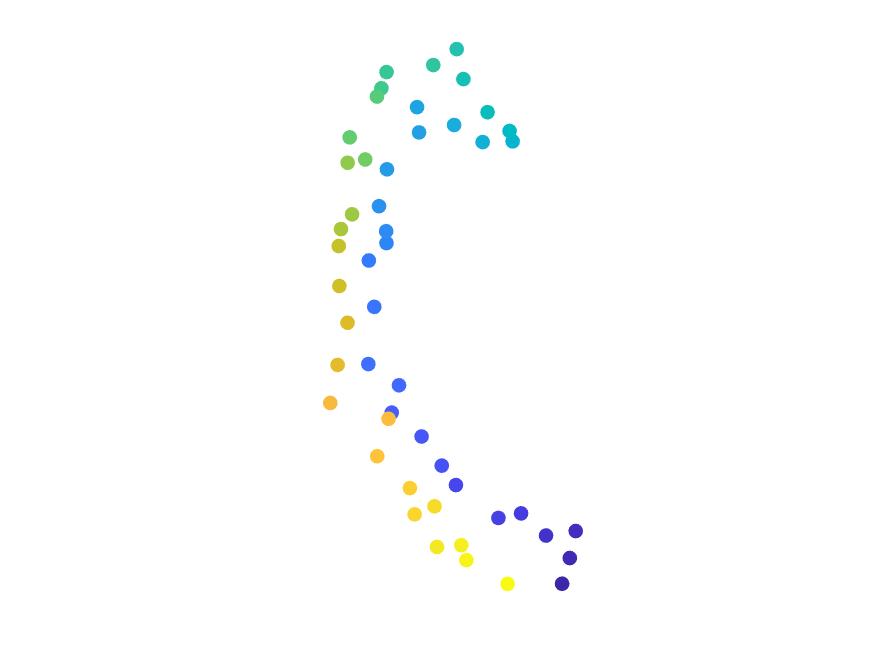}} &
            {\includegraphics[trim = 150 15 120 5, clip, height=.6in]{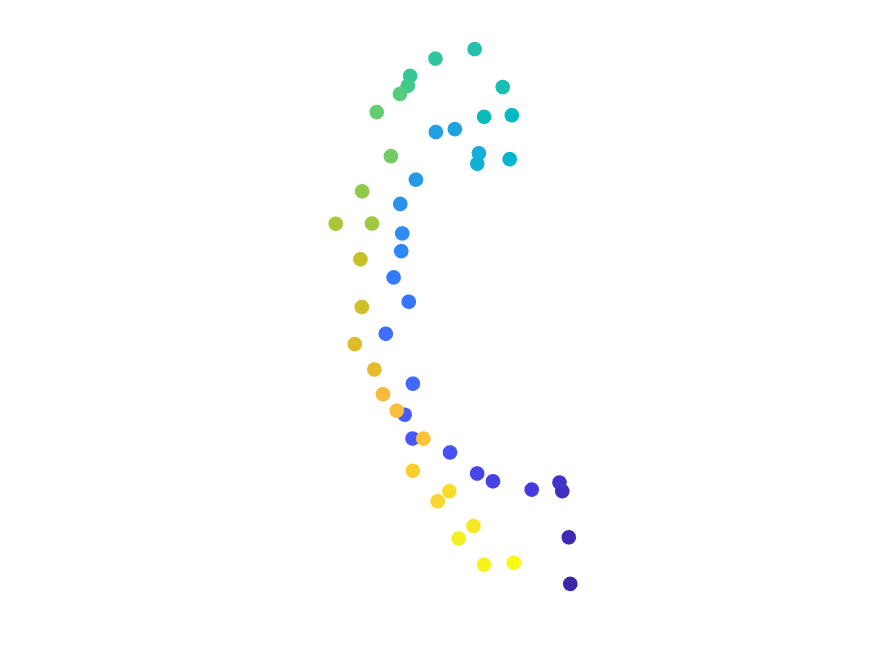}} &
            {\includegraphics[trim = 150 15 120 5, clip, height=.6in]{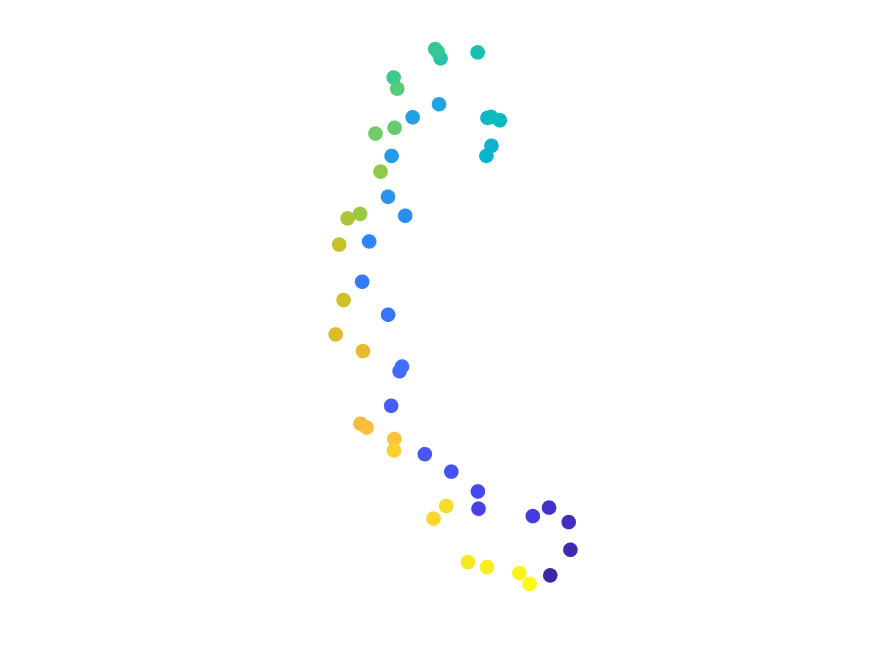}} \\\hline
            
            {\includegraphics[trim = 150 15 120 5, clip, height=.6in]{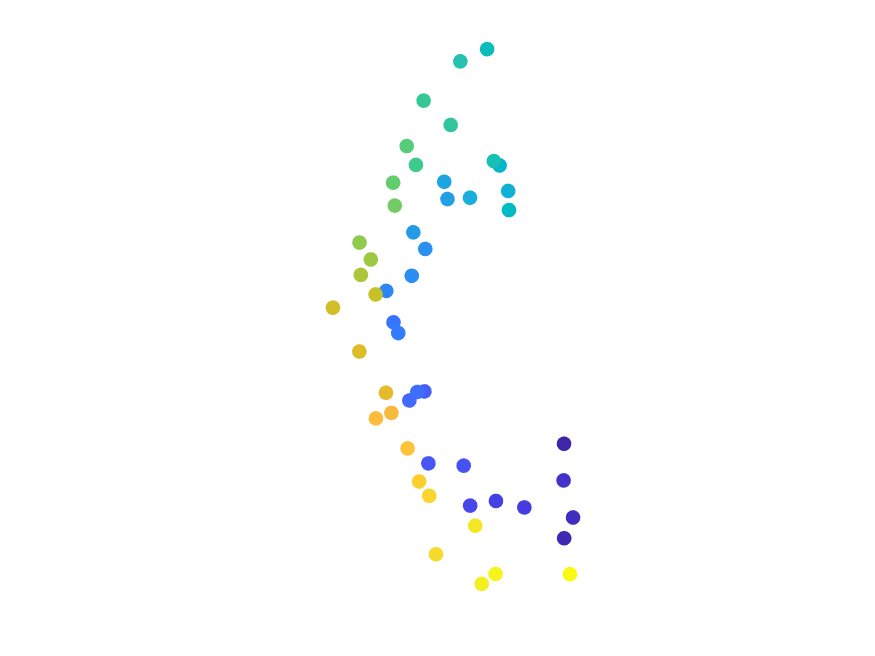}} &
            {\includegraphics[trim = 150 15 120 5, clip, height=.6in]{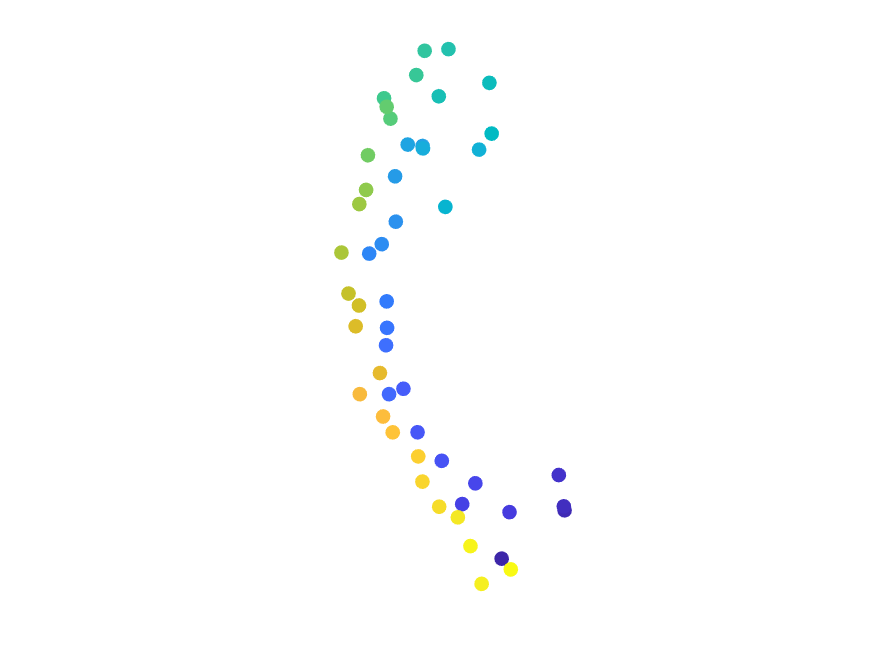}} &
            {\includegraphics[trim = 150 15 120 5, clip, height=.6in]{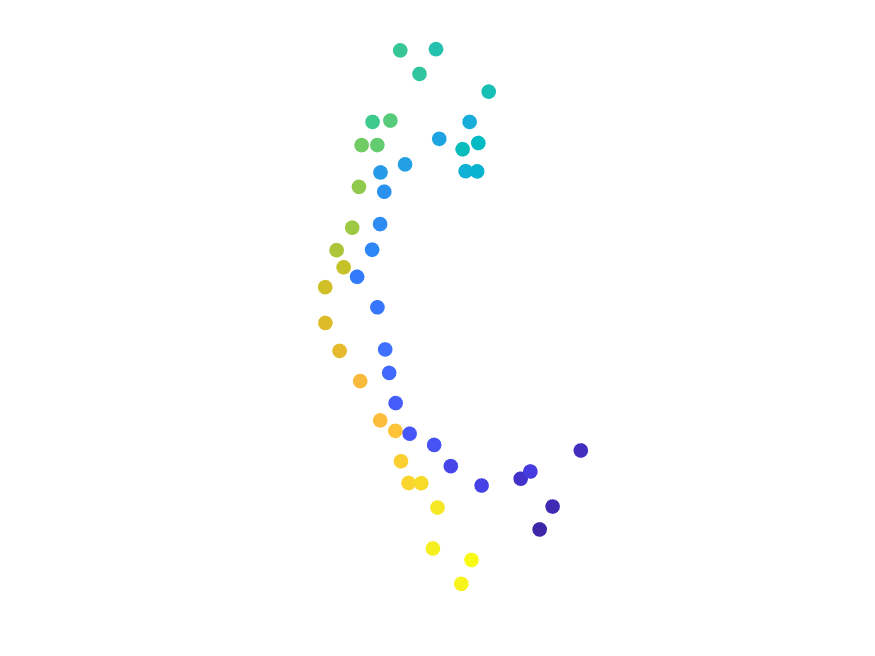}} &
            {\includegraphics[trim = 150 15 120 5, clip, height=.6in]{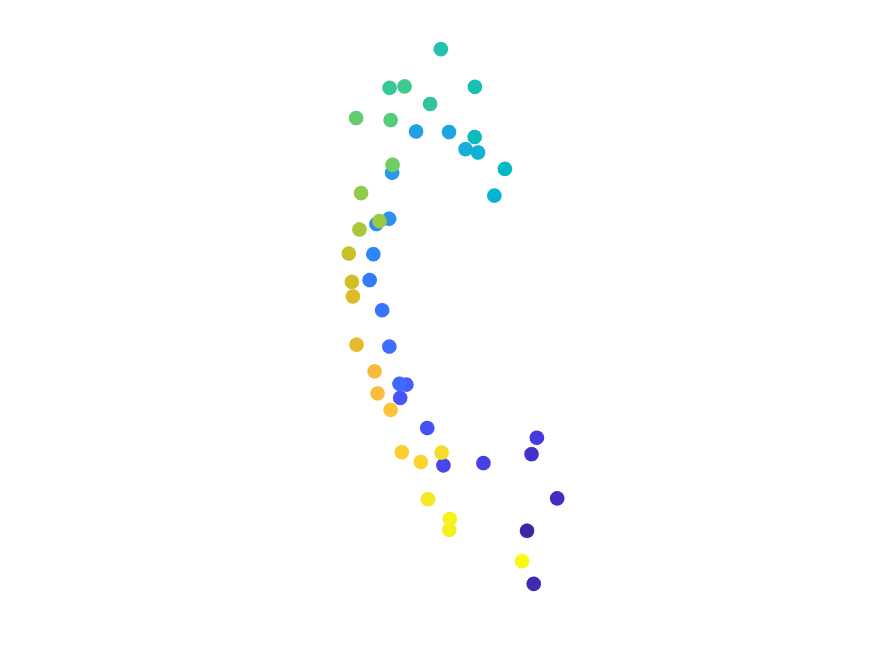}} &
            {\includegraphics[trim = 150 15 120 5, clip, height=.6in]{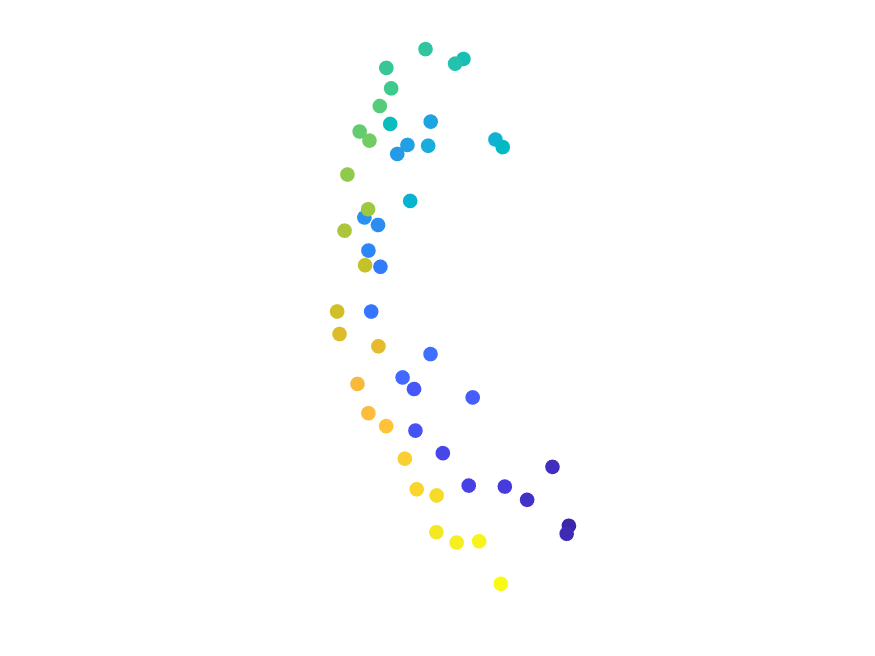}} &
            {\includegraphics[trim = 150 15 120 5, clip, height=.6in]{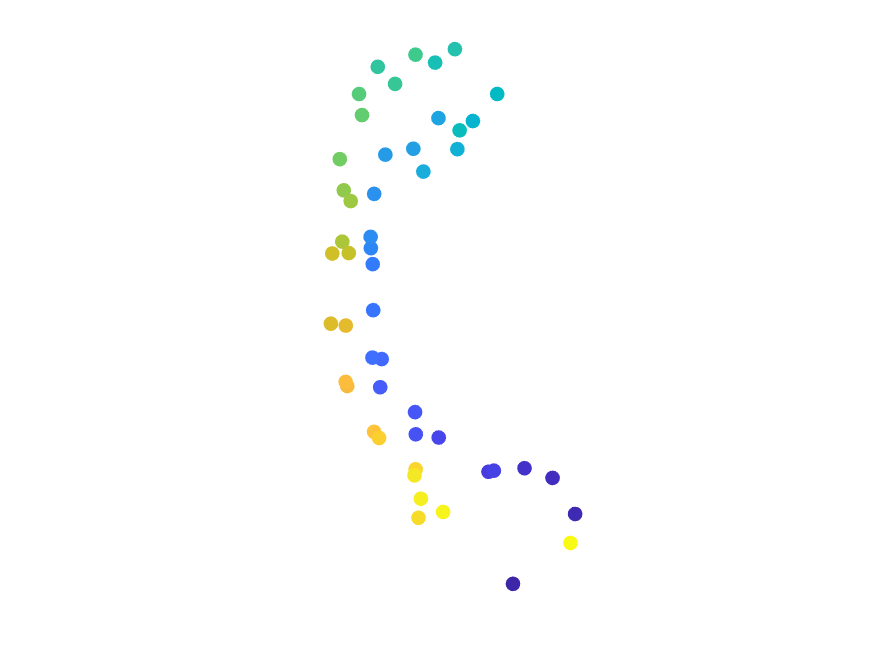}}\\\hline
            
        \fbox{\includegraphics[trim = 150 15 120 5, clip, height=.6in]{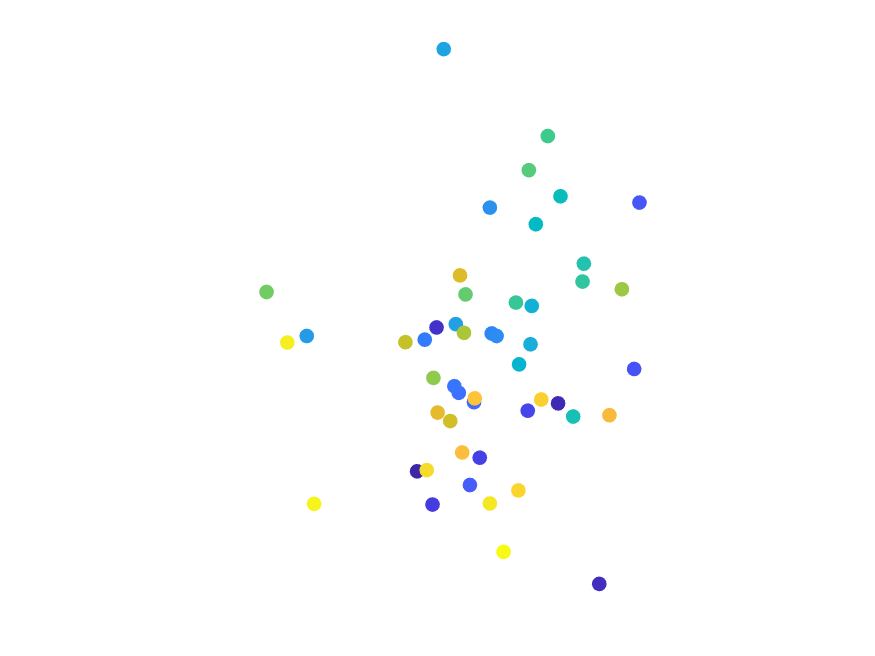}} &
        \fbox{\includegraphics[trim = 150 15 120 5, clip, height=.6in]{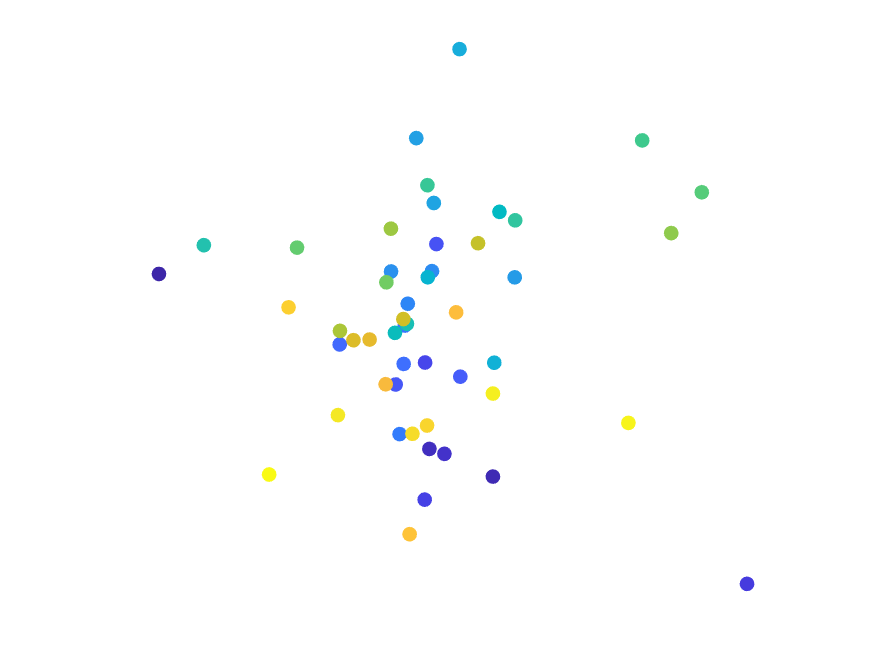}} &
        \fbox{\includegraphics[trim = 150 15 120 5, clip, height=.6in]{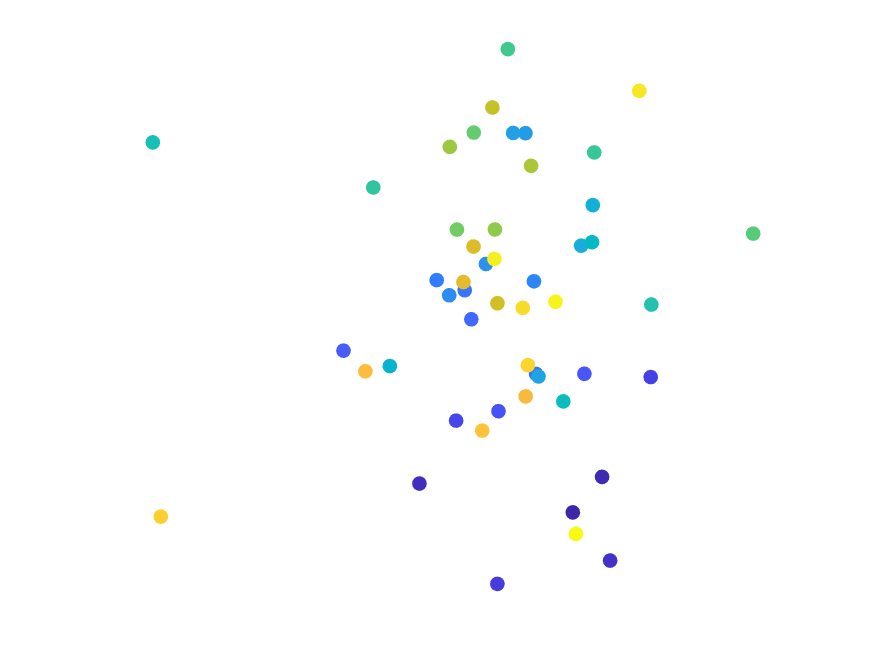}} &
        \fbox{\includegraphics[trim = 150 15 120 5, clip, height=.6in]{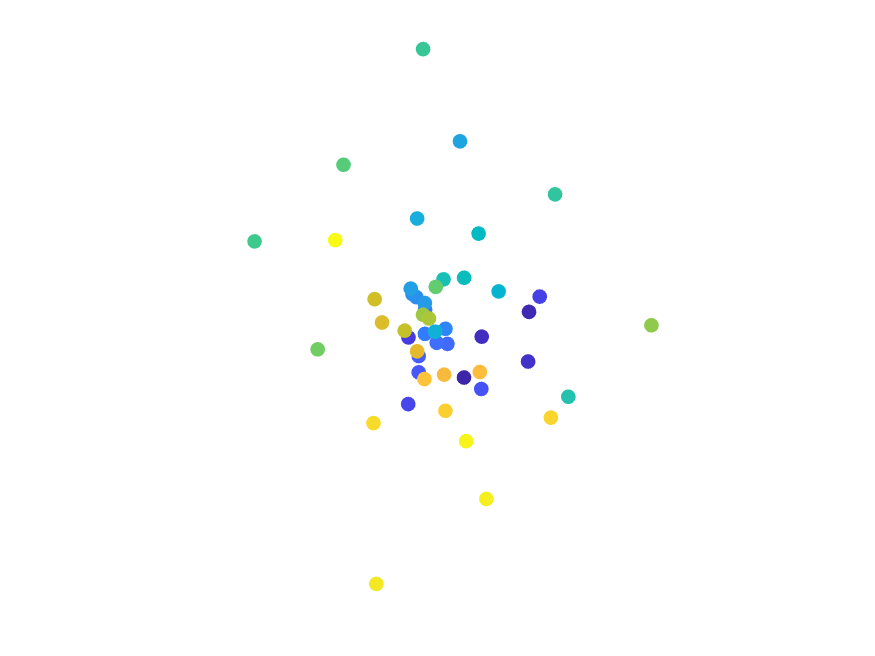}} &
        \fbox{\includegraphics[trim = 150 15 120 5, clip, height=.6in]{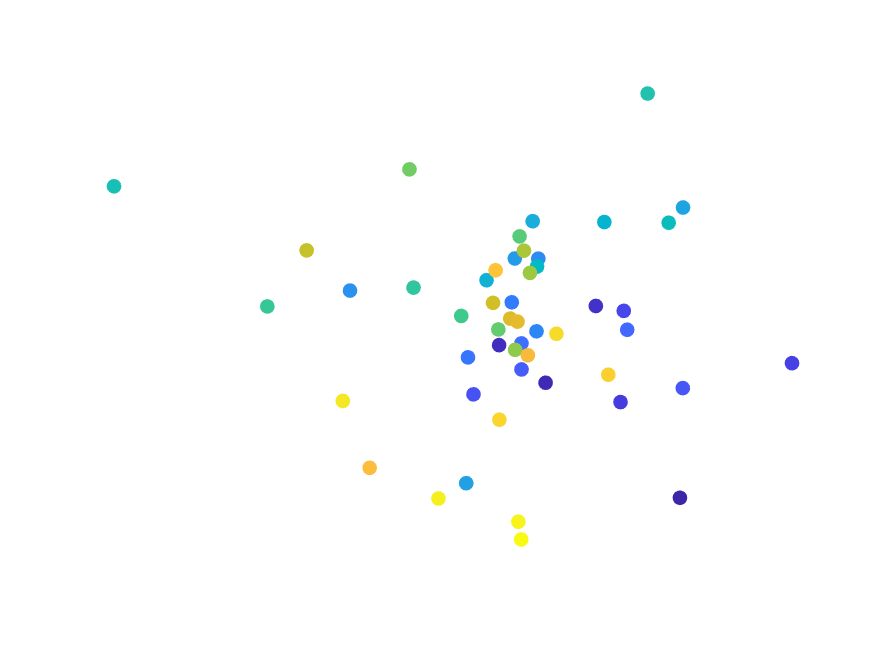}} &
        \fbox{\includegraphics[trim = 150 15 120 5, clip, height=.6in]{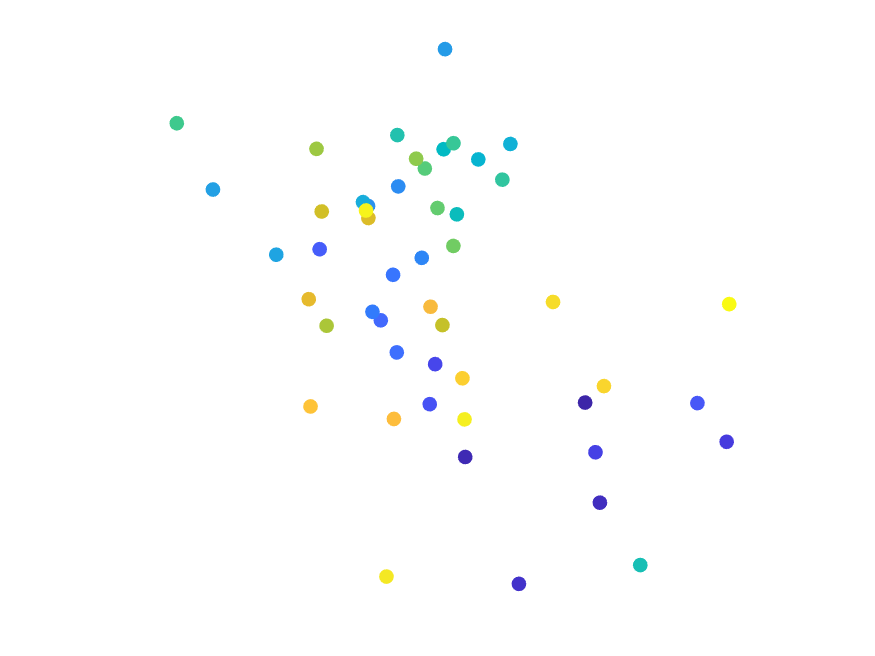}}\\\hline
            
        \end{tabular}

    \end{tabular}
    \caption{Left panel: A sample of ten corpus callosa from the sample of 409. Middle: The Fr\'echet mean corpus callosum. Right, top row: Six sample private corpus callosa privatized under the KNG framework on Kendall's 2D shape space. Right, middle row: Six sample private corpus callosa privatized point-wise with the Laplace mechanism. Right, bottom row: Six sample private corpus callosa privatized point-wise with the Laplace without accounting for rotational alignment. For more details, please refer to \ref{ss:K2D}.}
    \label{fig:CCs}
\end{figure}

As an application we consider the pre-processed corpus callosum data of \cite{cornea2017regression} from the Alzheimer's Disease Neuroimaging Initiative (ADNI). The data are from the mid-sagittal slices of MRIs (magnetic resonance images) and we refer to \cite{cornea2017regression} on details of how the data was processed. Their data contains 409 total corpus callosa. The left portion of Figure \ref{fig:CCs} displays 10 sample corpus callosa where the parameterization is visually displayed as a color gradient from blue to yellow. In the middle we display the Fr\'echet mean of all 409 corpus callosa. Having computed the mean, we then sanitize the mean with three techniques: (i) using the proposed KNG mechanism (Right: top row); (ii) for a comparison with (i), sanitize each landmark of the mean using the Laplace mechanism  splitting the privacy budget (Right: middle row); (iii) sanitize each landmark as in (ii) without factoring in rotational alignment in the corpus callosum. We expand on (ii) and (iii) in the Supplemental material. 
\begin{figure}
    \centering
    \begin{tabular}{|@{}c@{}@{}c@{}@{}c@{}@{}c@{}@{}c@{}@{}c@{} |@{}c@{}@{}c@{}@{}c@{}@{}c@{}@{}c@{}@{}c@{}|}
    \hline

        \includegraphics[trim = 120 0 130 0, clip,height = 0.8in]{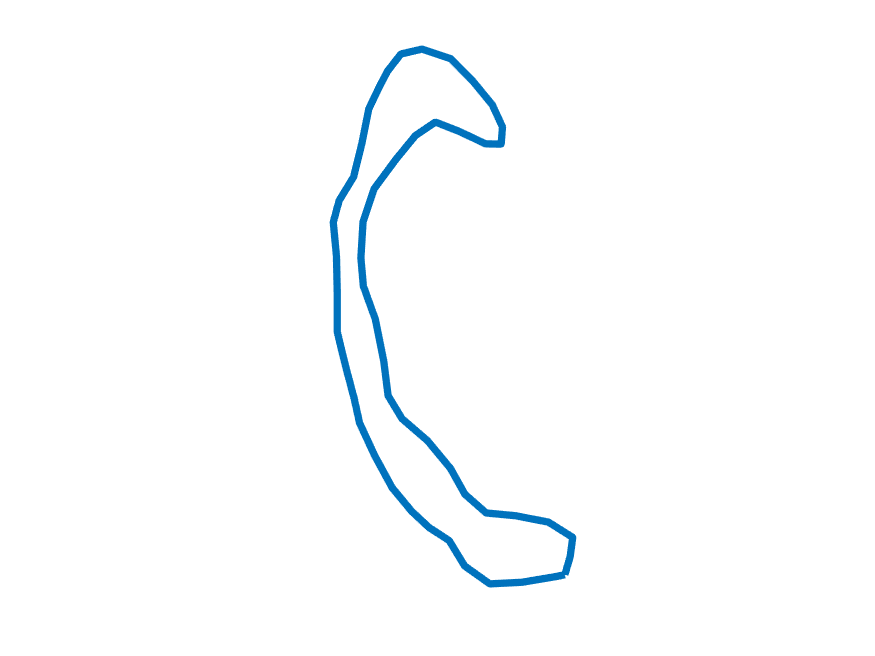} &
        \includegraphics[trim = 120 0 130 0, clip,height = 0.8in]{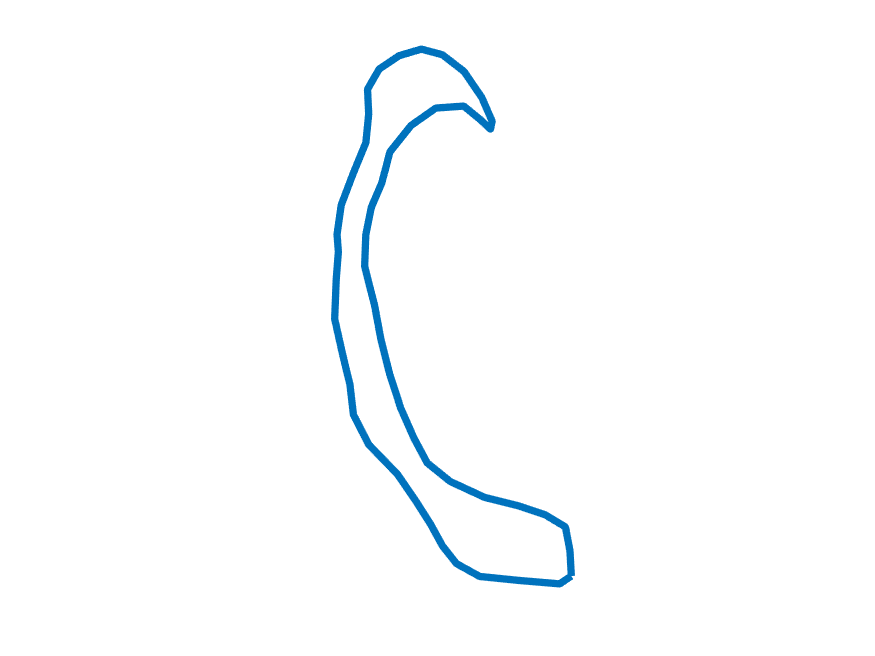} &
        \includegraphics[trim = 120 0 130 0, clip,height = 0.8in]{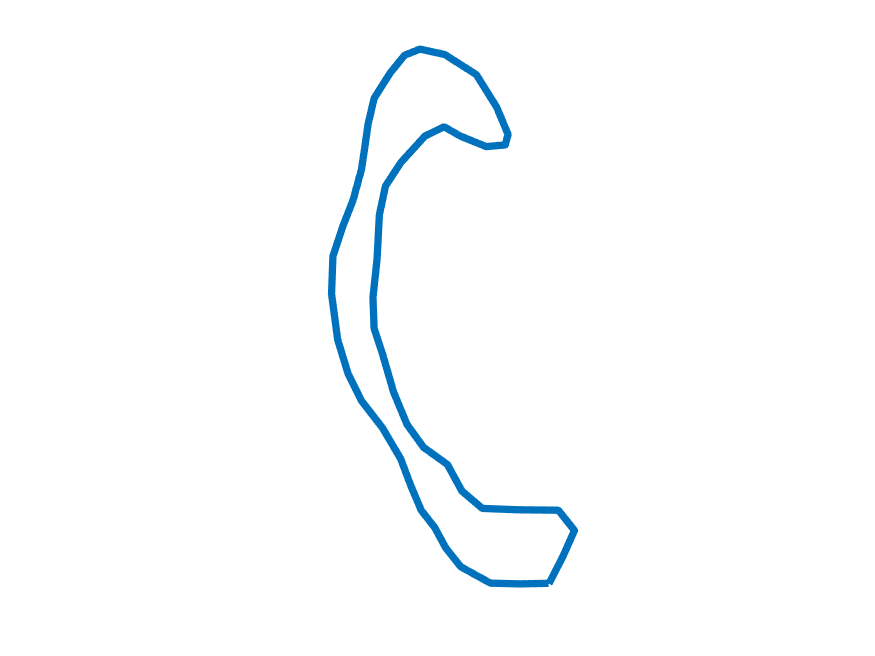} &
        \includegraphics[trim = 120 0 130 0, clip,height = 0.8in]{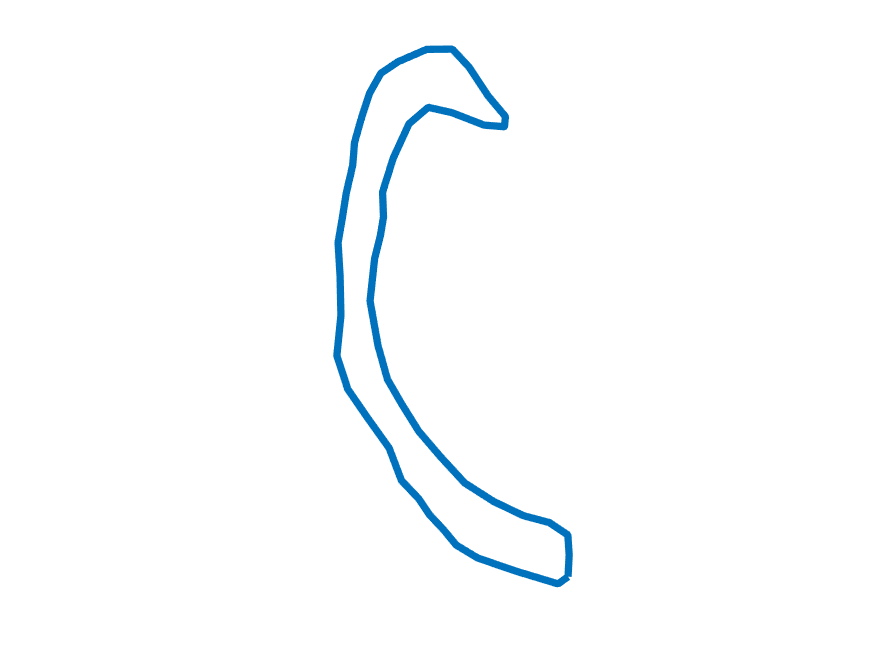} &
        \includegraphics[trim = 120 0 130 0, clip,height = 0.8in]{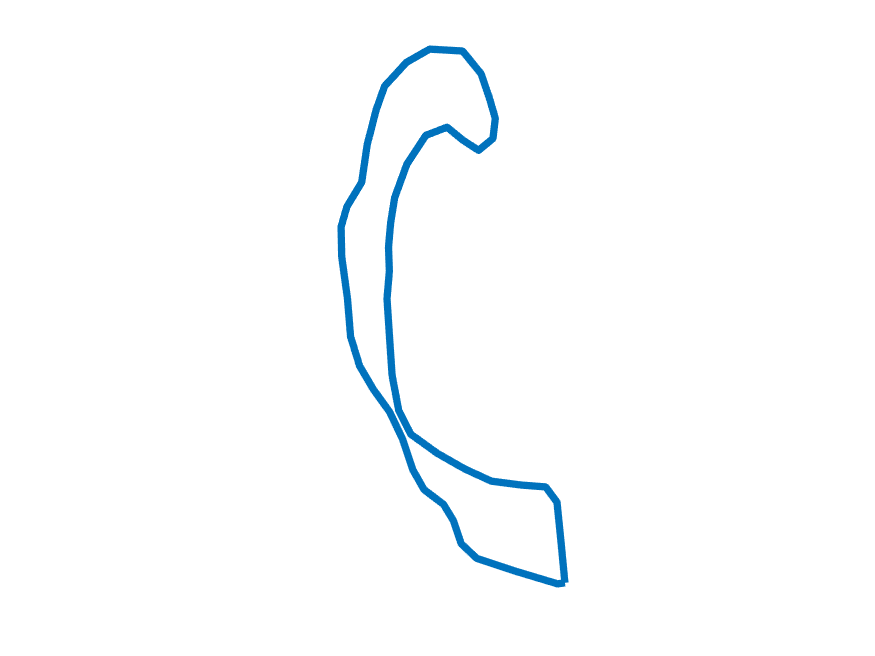} &
        \includegraphics[trim = 120 0 130 0, clip,height = 0.8in]{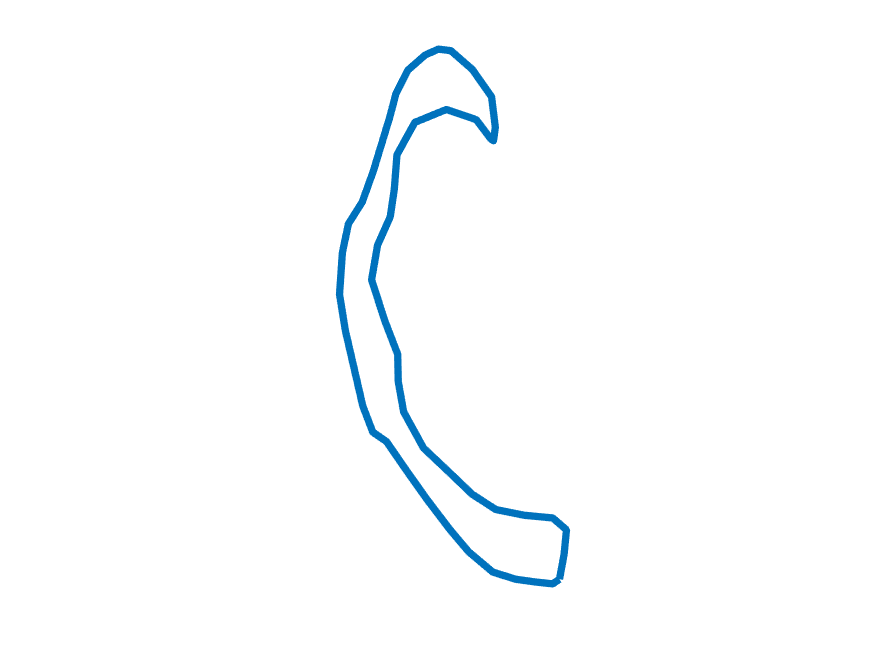} &
        \includegraphics[trim = 120 0 130 0, clip,height = 0.8in]{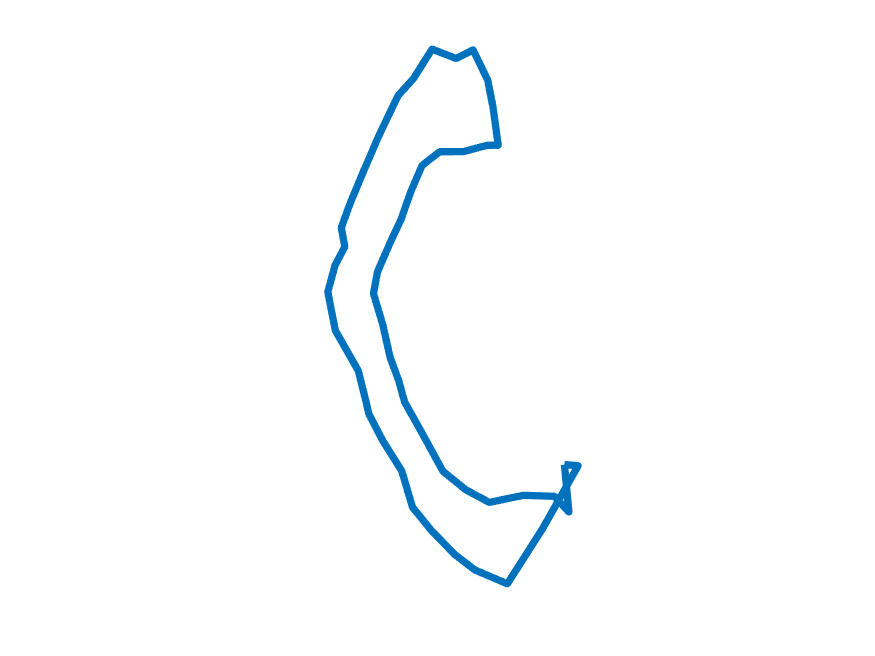} &
        \includegraphics[trim = 120 0 130 0, clip,height = 0.8in]{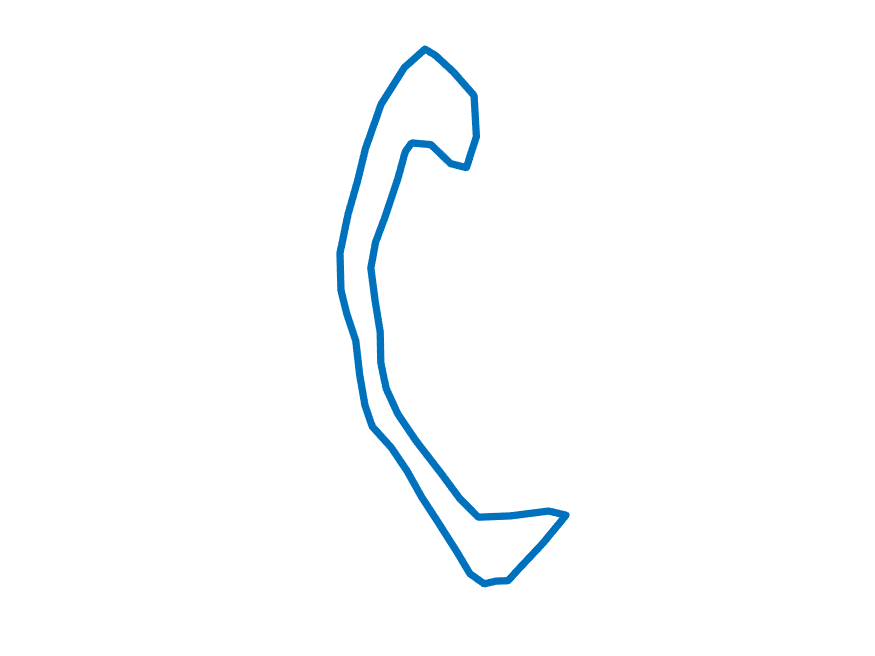} &
        \includegraphics[trim = 120 0 130 0, clip,height = 0.8in]{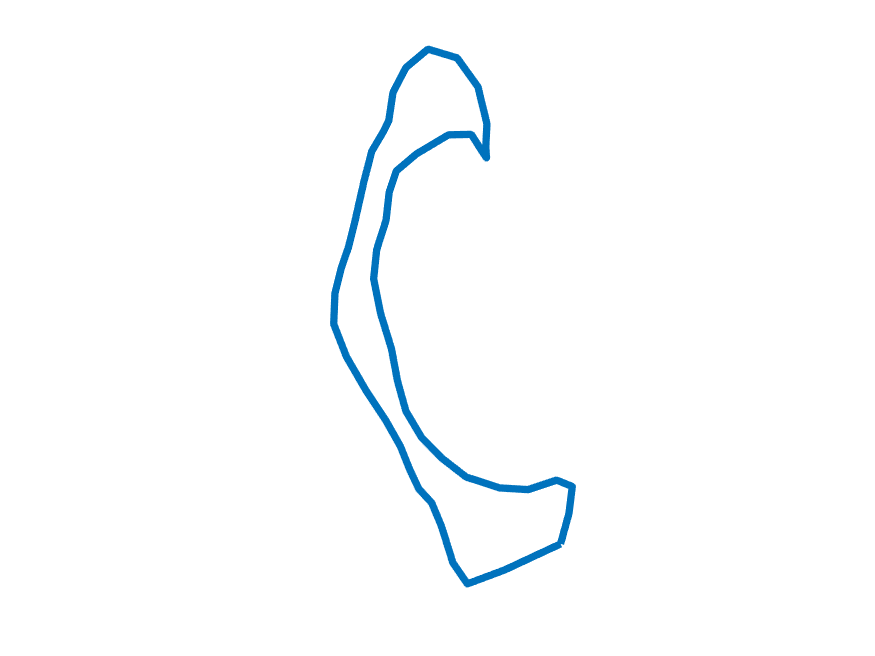} &
        \includegraphics[trim = 120 0 130 0, clip,height = 0.8in]{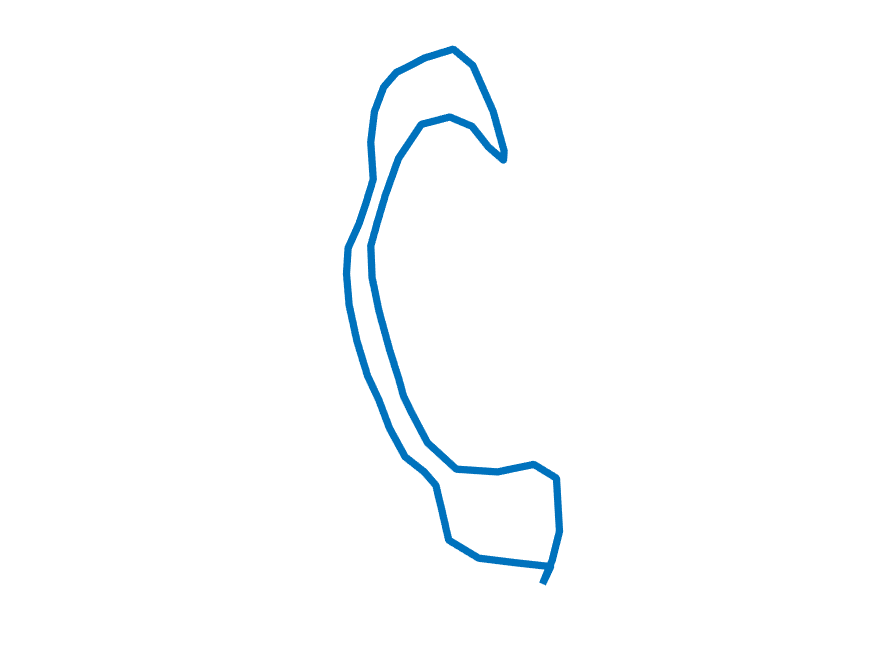} &
        \includegraphics[trim = 120 0 130 0, clip,height = 0.8in]{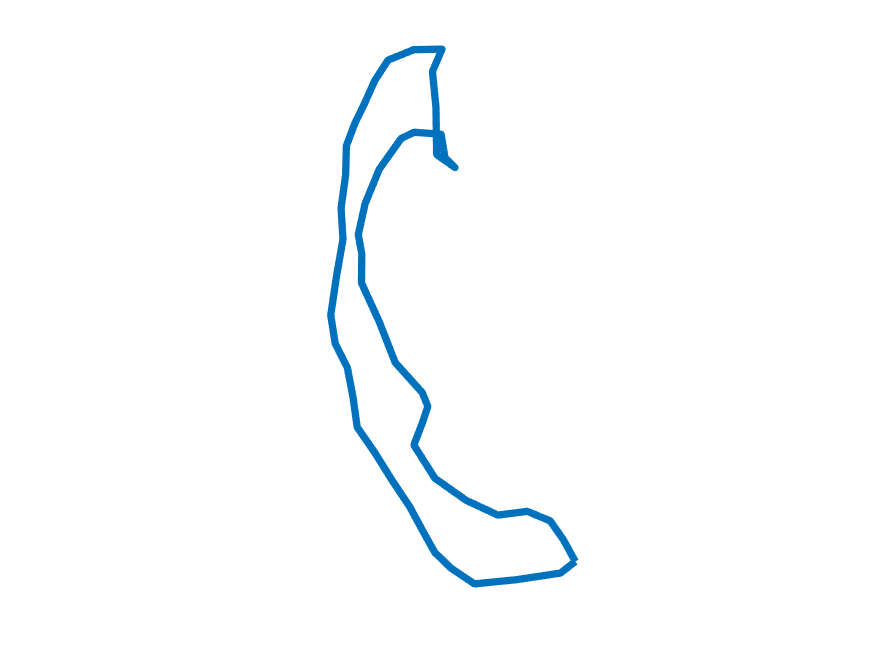} &
        \includegraphics[trim = 120 0 130 0, clip,height = 0.8in]{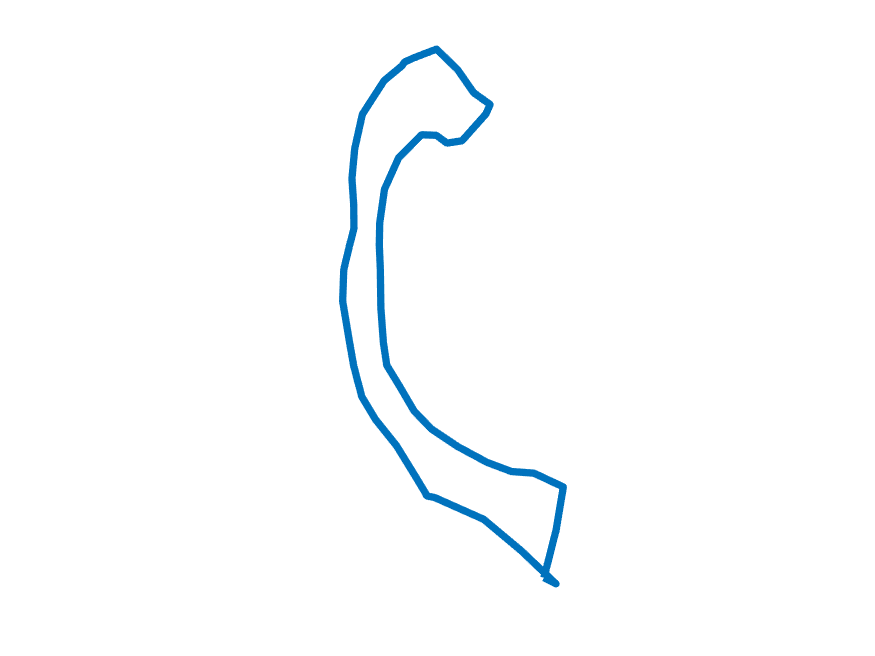} \\\hline
    \end{tabular}
    \caption{Differentially private corpus callosa estimates post-processed by using a first order local linear regression for smoothness. The left six are private under the proposed KNG framework while the right six are private under the point-wise Laplace as explained in Section \ref{ss:K2D}.}
    \label{fig:processed}
\end{figure}

As we can observe in Figure \ref{fig:CCs}, when one does not account for shape-preserving transformations (rotations alignment in this example), the shape of the corpus callosum is entirely destroyed during sanitization. Sanitization over the shape manifold (Right: top row) tends to retain the structure of corpus callosum when compared with sanitizing over each coordinate (Right: middle row), which appears more distorted; for example, we observe that the shape of the fifth corpus callosum (Right: middle row) has a contour with crossings. 
Further, we post-process to smooth the private estimates by a first order local linear regression. Left panel of Figure \ref{fig:processed} displays the post-processed KNG private estimates while the right panel shows point-wise Laplace private estimates, where rotational alignment has been carried out. We do not post-process the private estimates which do not consider rotational alignment as they appear to be non informative. Considering the mean shape in Figure \ref{fig:CCs}, the ``hook" at the top is quite prominent and comparing this to the private corpus callosa of Figure \ref{fig:processed} we notice that the KNG estimates tend to preserve this structure. Even though all the private estimates are processed in the same way, we notice the Laplace estimates are not only less smooth but can possess undesirable distortions, such as additional features such as loops. 
\section{Conclusions and future work}\label{sec:conclusions}
In this paper we demonstrate that versatility and powerful utility of the K-norm gradient mechanism on $\mathbb R^d$ carries over to the manifold setting. In particular, better control over global sensitivity when compared to the recently introduced manifold Laplace mechanism \citep{reimherr2021differential} for positively curved manifolds motivates the development of, to our knowledge, the first privacy mechanism for statistical shape analysis of 2D point configurations. Gains in utility when working directly on the manifold, as opposed to the higher-dimensional ambient space, are observed in the numerical examples: in terms of utility, the KNG not only outperforms the Euclidean mechanism but also the manifold Laplace mechanism. Further, the Laplace on the manifold and our mechanism are intricately connected as they both are the exponential of a norm in a particular tangent space as we note in Remark \ref{remark1}; this similarity in formulation and difference in sensitivity is tied to the better utility in the case of positively curved manifolds. 

Depending on the manifold, statistical utility gains enjoyed by working on the manifold can be tempered by expensive geometric computation. For example, in the case of SPDM, the clear gains in utility are obtained within the context of computationally expensive sampling from the KNG (see also Supplemental material) owing to repeated computations of matrix inverses and square roots related to the exponential, inverse-exponential maps and the geodesic distance. Indeed in practice, however, only a single instantiation suffices. In contrast, sampling from the Laplace on manifolds is straightforward, and there is thus room for improvement in sampling from the KNG on manifolds. 

Our work represents a step in the right direction in developing privacy mechanisms for a myriad of approaches to state-of-the-art statistical shape analysis on infinite-dimensional manifolds of curves and surfaces \citep{srivastava2016functional}, and diffeomorphisms  \citep{GrenanderMiller}. Moreover, our work opens up the possibility of developing geometry-driven privacy mechanisms for standard data analytic procedures used in various applications, such as principal component analysis (Grassmannian manifold of subspaces), rank-constrained matrix completion (quotient manifold of nonsingular matrices), and optimizing the Rayleigh quotient (Grassmannian), Procrustes problem (manifold of orthogonal matrices or frames), and pose estimation in computer vision (manifold of rotation matrices).


\begin{ack}
This work was funded in part by NSF SES-1853209.
\end{ack}

\bibliographystyle{abbrvnat}
\bibliography{references.bib}



\appendix
\LARGE{Supplemental to {\it Shape and Structure Preserving Differential Privacy}}
\nopagebreak
\normalsize

\section{Proof of Lemma 1}
Recall that $f:\mcM \to \mathbb R$ is strong geodesically convex with parameter $\lambda>0$ if for every $x,y \in \mcM$
\[
f(y) \geq f(x)+\left \langle \nabla f(x),\exp^{-1}(x,y) \right \rangle_x
+\frac{\lambda}{2} \rho^2(x,y);
\]
the above definition is interpreted in a local sense within a suitable neighbourhood in which the inverse exponential map is well-defined. We first establish that $x \mapsto U(x, D) $ is strong geodesically convex and derive upper and lower bounds on its Hessian within $B_{r}(p_0)$ with $r$ chosen as per Assumption 1. 

Note that when $s < \frac{\pi}{2 \sqrt{\kappa_{\max}}}$, the function $s \mapsto h_{\max}(s,\kappa_{\max})>0$, decreasing, and bounded above by 1, while $s \mapsto h_{\min}(s,\kappa_{\min})$ is bounded below by 1 and increasing. Thus $0<h_{\max}(2r,\kappa_{\max})\leq h_{\max}(\rho(x,y),\kappa_{\max})$ for every $x,y \in B_{r}(p_0)$ under the assumption on $r$. On the other hand, $h_{\min}(2r,\kappa_{\min})\geq h_{\min}(\rho(x,y),\kappa_{\min}) \geq 1$ for every $x,y \in B_r(m_0)$. For each $x_i \in B_{p_0}(r)$, from Lemma 1 of \cite{Alimisis2020ACP} we have that for $x,y \in B_r(m_0)$,
\begin{align}
\label{strong}
\rho^2(x,x_i) &\geq \rho^2(y,x_i)+\langle \nabla \rho^2(y,x_i),\exp^{-1}(y,x) \rangle_y+\frac{2h_{\max}(\rho(x,x_i),\kappa_{\max})}{2}\|\exp^{-1}(y,x)\|_y^2 \nonumber \\
& \geq \rho^2(y,x_i)+\langle \nabla \rho^2(y,x_i),\exp^{-1}(y,x) \rangle_y+\frac{2h_{\max}(2r,\kappa_{\max})}{2}\|\exp^{-1}(y,x)\|_y^2.
\end{align}
Summing over $i$ and dividing by $2n$ we get
\[
U(x,D) \geq U(y,D)+\langle \nabla U(y,D),\exp^{-1}(y,x) \rangle_y+\frac{h_{\max}(2r,\kappa_{\max})}{2}\|\exp^{-1}(y,x)\|_y^2 ,
\]
which implies that $U$ is strong geodesically convex with parameter $h_{\max}(2r,\kappa_{\max})$ inside $B_{r}(p_0)$. Thus, in local coordinates
\[
\nabla^2U(x,D) \succcurlyeq h_{\max}(2r,\kappa_{\max})\mathbb I_d.
\]

From assumption on the lower bound $\kappa_{\min}$ on sectional curvatures ensures, using Lemma 5 of \cite{Zhang2016FirstorderMF} derived for non-positively curved manifolds, we obtain
\[
U(x,D) \leq U(y,D)+\langle \nabla U(y,D),\exp^{-1}(y,x) \rangle_y+\frac{h_{\min}(2r,\kappa_{\min})}{2}\|\exp^{-1}(y,x)\|_y^2 .
\]
In other words, $U$ has a gradient that is geodesically Lipschitz with parameter $h_{\min}(2r,\kappa_{\min})$. As a consequence,
\[
\nabla^2U(x,D) \preccurlyeq h_{\min}(2r,\kappa_{\min})\mathbb I_d.
\]
On flat manifolds (e.g., $\mathbb R^d$, flat torus, cylinder) where $\kappa_{min}=\kappa_{\max}=0$ we have  $h_{\min}(2r,0)=h_{\max}(2r,0)=1$, and $U$ is strong gedesically convex with parameter 1 when $D$ is restricted to lie within a ball of any finite radius. Summarily, in local coordinates, 
\begin{equation}
\label{bounds}
h_{\max}(2r,\kappa_{\max})\mathbb I_d \preccurlyeq \nabla^2 U(x,D) \preccurlyeq h_{\min}(2r,\kappa_{\min}) \mathbb I_d, \quad \forall x \in B_r(m_0).
\end{equation}
We now consider the norm of the gradient vector field $\nabla U$. 
Let $\gamma$ be a unit-speed geodesic from $\bar x$
 to $x$, and denote by $\Gamma_{\bar x}^x:T_{\bar x}M \to T_x M$ the parallel transport along $\gamma$; from our assumption on the radius $r$, the geodesic lies entirely within $B_{r}(p_0)$. Then,
 \begin{align*}
 \|\nabla U(x,D)\|_x&=\big|\|\nabla U(x,D)\|_x-\|\Gamma_{\bar x}^x \nabla U(\bar x,D)\|_x \big|\\
  &\leq \|\nabla U(x,D)-\Gamma_{\bar x}^x \nabla U(\bar x,D)\|_x\\
  & \leq h_{\min}(2r,\kappa_{\min})\rho(\bar x,x) \thickspace.
 \end{align*}
The equality is due to the fact that the parallel transport map is an isometry between tangent spaces and fixes the origin; the first inequality follows from the reverse triangle inequality, while the last follows from the upper bound on the Hessian of $U$ in \eqref{bounds}.\\

To derive the lower bound on $\|\nabla U(x,D)\|_x$, note that under our assumption on the radius $r$, the function $U$ is strong geodesically convex within $B_{r}(p_0)$ since the function $h_{\max}$ is positive. From the lower bound on the Hessian of $U$ in \eqref{bounds}, for any $y \in B_{r}(p_0)$, we hence obtain
\[
\Big|\big\langle \nabla U(x,D)-\Gamma_y^x \nabla U(y,D),\thickspace \exp^{-1}(x,y) \big\rangle_x\Big| \geq h_{\max}(2r,\kappa_{\max})\rho(x,y)^2,
\]
where $\Gamma^x_y$ is the parallel transport along a geodesic from $y$ to $x$; applying Cauchy-Schwarz to the inner product results in
\[
\|\nabla U(x,D)-\Gamma_y^x \nabla U(y,D)\|_x \rho(x,y) \geq 
h_{\max}(2r,\kappa_{\max})\rho(x,y)^2.
\]
Dividing both sides by $\rho(x,y)$ and taking $y=\bar x$ leads to the 
desired lower bound, since, within the local coordinates at $x$, $\Gamma_{\bar x}^x \nabla U(\bar x, D)$ is the zero gradient vector field under the isometric parallel transport. 

\section{Simulation details}
Simulations pertaining to the sphere and Kendall shape space are done on a desktop computer with an Intel Xeon processor at 3.60GHz with 31.9 GB of RAM running windows 10. Simulations pertaining to symmetric positive-definite matrices were performed on the Pennsylvania State University’s Institute for Computational and Data Sciences’ Roar supercomputer. All simulations are done in Matlab. This content is solely the responsibility of the authors and does not necessarily represent the views of the Institute for Computational and Data Sciences.
\subsection{Fr\'echet Mean}
To compute the Fr\'echet mean, we use the standard gradient descent approach. Given a sample $D=\{x_1,x_2,\dots,x_n\}$ we initialize $\hat{\mu}_1$ at a data point. The variance functional which we wish to minimize is $U(x;D)=-\frac{1}{2n}\sum_{i=1}^n\rho^2(x,x_i)$. 
So, at iteration $k$ one takes a step in the direction of $t_k\nabla U(\hat{\mu}_{k-1};D)$ where $t_k\in(0,1]$ from $\mu_{k-1}$. That is, $\hat{\mu}_k=\exp (\hat{\mu}_{k-1},t_k\nabla U(\hat{\mu}_{k-1};D))$. To check convergence one could either check if the distance between adjacent iterations is smaller than some value $\beta_1>0$ or if the norm of the gradient is smaller than some value $\beta_2>0$, that is if $\rho(\hat{\mu}_1,\hat{\mu}_2)<\beta_1$ or if $\|\nabla U(\hat{\mu}_k;D)\|_{\hat{\mu}_k}<\beta_2$. The latter is convenient from a computational standpoint and how we measure convergence. We set $\beta_2=10^{-5}$ and $t_k=0.5$ for all $k$. Further, to avoid a computational timeout we set a maximum number of iterations to 500 however for all examples the algorithm converged within the first couple hundred iterations. Lastly, we assume the data follows Assumption \ref{A1} and thus convergence issues and local minima pose no observed issues.
\subsection{The Euclidean Laplace}
A standard distribution to generate differentially private estimates is the Euclidean Laplace which is a K-norm mechanism with the $\mathbb{\ell}_2$ norm. We sample from the distribution  $f(x)\propto \exp\{-\sigma^{-1}|x-\bar{x}|\}$ on $\mbR^d$ as in \cite{reimherr2021differential}.
\begin{enumerate}
    \item Sample a direction $V$ uniformly from $\mcS^{d-1}$.
    \item Sample a radial length $r$ from $\Gamma(d,1)$, the Gamma distrbution with $\alpha=d$ and $\beta=1$.
    \item Set $Y=\bar{x}+r\sigma V$.
\end{enumerate}
$Y$ will then be a draw from $f(x)$, the $d-$dimensional Euclidean Laplace distribution. Note that $x/|x|$ is uniform on $\mcS^{d-1}$ when $x \sim N_d(\bm 0_d,\mathbb I_d)$.
\subsection{SPDM simulations}
\subsubsection{Generating random samples}
We have that $\mbPk$ is the space of symmetric positive-definite matrices. We note that the Wishart distribution has support on $\mbPk$, and draw from $Y\sim W(V,df)$ where $E(Y)=V\cdot df$, $df>0$, and $V$ is a symmetric $k\times k$ matrix.
We require that the $D\subset B_r(p_0)$ but note there is non canonical choice for $p_0$ or $r$. We set $p_0$ to be the identity matrix $I_k$, $V=\frac{1}{k}I_k$, and $df=k$. Recall that since $\mbPk$ is negatively curved under the chosen metric, $r$ is finite but unconstrained. Operationally, however, there is no reason to believe that $\rho(y,I_k)\leq r$ for any chosen $r$, where the distance is manifold distance; we thus first set an $r$ and discard draws which are greater than distance $r$ from $p_0$ until we have sufficient draws for a desired sample size. For the simulations we set $k=2$ and $r=1.5$.
\subsubsection{Sampling from KNG on SPDM}\label{ss:KNGSPDM}
To sample from KNG for mean estimation on SPDM we use Metropolis-Hastings, a Markov chain Monte Carlo method. Let $\vech(\cdot)$ denote the vectorization of a symmetric matrix and $\vech^{-1}(\cdot)$ denote its inverse. Recall that the dimension $k=2$. At each iteration $i$ we generate a proposal $x'$ by first randomly drawing a matrix $v$ from the tangent space at the current stage $T_{x_i}\mathbb P(2)\cong Sym_2$, and moving along $\mathbb P(2)$ in the direction of $v$ using the exponential map by proposing $f(x'|x_i)=\exp(x_n,t\sigma v)$ where $t\in(0,1]$. We sample $v$ as $v=\vech^{-1}(\tilde{v})$, where $\tilde v$ is $k(k+1)/2=3$-dimensional vector of uniform random variables on $[-0.5,0.5]$; this indeed is not the same as a uniform draw on $Sym_2$. 
We then accept or reject the proposals of $f$ producing a Markov chain for the density $g(x)\propto \exp\{-\|U(x;D)\|_x/\sigma\}$. Here $\sigma=2\Delta/\epsilon=4r/n$ since $\epsilon=1$ and $\Delta$ is as in Theorem \ref{thm:sensitivity}.

\begin{enumerate}
    \item Initialize $x_0=\bar{x}$.
    \item At the $i$th iteration, draw a matrix $v\in Sym_k$, the tangent space of $x_i$, as described above.
    \item Generate a proposal $x'$ by letting $x'=\exp(x_i,t\sigma v)$.
    \item Accept $x'$ and set $x_{i+1}=x'$ with probability $g(x')/g(x_i)$. Otherwise, reject $x'$ and generate a new candidate by returning to the previous generation step.
    \item Return to step 2 until a chain of sufficient length has been created.
\end{enumerate}
For our simulations we tuned $t$ at each sample size, but had a minimal 5000 burn-in steps and a thinning jump width of approximately 5000 to avoid correlation between adjacent accepted samples. 
\subsubsection{Choosing the ambient space radius}
We compare privatization of our mechanism over $\mbPk$ to privatization in the ambient space of symmetric matrices $Sym_k$. To do this, we need to compute a comparable sensitivity. That is, given our data $D\subset B_r(I_k)\subset\mbPk$ we need to find $r_E$ of $D\subset\mcB_{r_E}(I_k)\subset Sym_k$, the radius of the geodesic ball in the space of symmetric matrices. Given the ball is centered at the identity matrix, it turns out that $r_E=e^r-1$ as shown in \cite{reimherr2021differential}.
\subsection{Sphere simulations}
\subsubsection{Generating random samples}
To generate random samples in $B_r(p_0)\subset\mcS_1^2$ we use polar coordinates. First, let $(\theta,\phi)$ be the pair of angles where $\theta\in[0,\pi]$ is the radial coordinate and $\phi \in [0,2\pi)$ is the polar angle. We uniformly sample on $\theta\in[0,r]$ and $\phi\in[0,2\pi)$ and set $r=\pi/8$. This results in data in $B_r(p_0)$ where $p_0$ is the north pole and higher concentration of data nearer $p_0$.
\subsubsection{Sampling from KNG on $\mcS_1^2$}
To sample from KNG on $ \mcS_1^2$ we use a Metropolis-Hastings algorithm in a manner similar to that described above in Section \ref{ss:KNGSPDM} with the only difference being how we make proposals. At each iteration $i$ we generate a proposal $x'$ in the following manner. First, we draw a sample direction by drawing a vector from $N_3(\bm 0_3,\mathbb I_3)$, scale this vector to have length $\sigma$, and then project this vector onto the tangent space $T_{x_i}\mcS_1^2$ of $x_i$ 
to produce a vector $\tilde{v}$. We then make a proposal by setting $x'=\exp(x_i,t\tilde{v})$ with $t\in(0,1]$. The projection onto the tangent space ensures that $\|\tilde{v}\|\leq \sigma$. 

Here $\sigma = \Delta/\epsilon$ where $\Delta = 2r(2-h(r,\kappa))/n$ and $\epsilon=1$. For our simulations we set $t=0.5$, a burn-in period of 20 000 and thinned the chain every 600 to avoid correlated adjacent samples.
\subsection{Kendall's 2D shape space simulations}
\subsubsection{KNG over Kendall's shape space}\label{ss:K2DKNG}
To sample from KNG on the space of Kendall shape space we use a Metropolis-Hastings algorithm similar to that described in Section \ref{ss:KNGSPDM}. Let $k$ be the number of landmarks for our set of shapes. At each iteration $i$ we generate a proposal $x'$ as follows:
\begin{enumerate}
    \item Sample $v=\{v_i\}$ such that the real and imaginary components of each $v_i$ are independent draws from  $U(0,1)$. 
    
    \item Set $\tilde{v}=v-\frac{1}{k}\sum v_i$, which is simply $v$ centered at the origin.
    \item Compute the horizontal component of $\tilde{v}$ on the tangent space of $x_i$, denote this as $\tilde{v}_h$.
    \item Propose $x'=\exp(x_i,t\tilde{v}_h)$.
\end{enumerate}
For our experiments we have a burn-in period of 7500.

We make two key assumptions that are needed for computing the sensitivity. First, we set $\kmax=4$ which is the maximal curvature of Kendall shape space. This perhaps can be improved upon, but we only require an upper bound on the curvature as per Assumption \ref{A1}, so we assume a worst case scenario. We set $r=\max_i \rho(\bar{x},x_i)$ the maximum shape distance from the Fr\'echet mean shape and all landmark configurations in our dataset.
\subsubsection{Shape point-wise Laplace}
Suppose we have a dataset $D=\{x_1,x_2,\dots,x_n\}$ such that $x_i=\{x_{i,j}\}\in \mbC^k$ is an set of $k$ labelled landmarks. We assume the shapes are all centered and scaled as in \ref{ss:K2D}. To compute a sanitized estimate in Euclidean space we sanitize each coordinate in the following manner. 
\begin{enumerate}
    \item Compute the Fr\'echet mean in shape space, denote this as $\bar{x}=\{\bar{x}_j\}$
    \item Rotationally align each shape $x_i$ to the mean $\bar{x}$ using Procrustes analysis, denote this as $\tilde{x}_i$. That is, $x_i\rightarrow Ox_i$ where $O=\argmin_{O\in SO(2)}\rho(\bar{x},Ox_i)$,
    \item For each landmark $j$, find the maximal distance from $\bar{x}_j$ in the real and imaginary direction, call these $d_{x,j}$ and $d_{y,j}$.
    \item Sanitize each landmark of the mean in the real and imaginary direction using the standard Laplace.
\end{enumerate}
Similar to Section \ref{ss:K2DKNG} above, rather than assume an $r$ for the ball in which the data lives, we determine this by setting this as a maximum distance to each landmark. So, at each landmark and in each direction we have $\sigma=\Delta/(\epsilon/2k)=4rk/n$ where we set $r=d_{x,j}$ for the real direction and $r=d_{y,j}$ in the imaginary direction. We divide the privacy budget by $2k$ since we sanitize each landmark in each coordinate. By dividing the privacy budget among the landmarks in this way, the entire shape will have total privacy budget $\epsilon$.
We compute the orthogonal alignment using standard Procrustes analysis.

\end{document}